\documentclass[letterpaper]{article}
\usepackage[margin=1in]{geometry}
\usepackage{times}

\usepackage[style=authoryear,backend=bibtex]{biblatex}
\usepackage{appendix}

\usepackage[utf8]{inputenc}
\usepackage[T1]{fontenc}
\usepackage[hidelinks]{hyperref}
\usepackage{url}
\usepackage{booktabs}
\usepackage{amsfonts}
\usepackage{nicefrac}
\usepackage{enumitem}
\usepackage{algorithm}
\usepackage{algpseudocode}

\usepackage{xcolor}
\usepackage{graphicx}
\usepackage{subcaption}
\usepackage{microtype}

\usepackage{amsmath,amssymb}
\usepackage{amsthm} %
\newtheorem{thm}{Theorem}%
\newtheorem*{thm*}{Theorem}
\newtheorem{prop}[thm]{Proposition}
\newtheorem{lem}[thm]{Lemma}
\newtheorem{cor}[thm]{Corollary}
\theoremstyle{definition}
\newtheorem*{ex}{Example}
\theoremstyle{definition}

\newcommand{\Erdos}{Erd\H{o}s}
\newcommand{\Renyi}{R\'enyi}
\renewcommand{\cite}[2][]{\autocite[#1]{#2}}

\newcommand{\C}{\mathbb{C}}
\newcommand{\R}{\mathbb{R}}
\newcommand{\Q}{\mathbb{Q}}
\newcommand{\Z}{\mathbb{Z}}
\newcommand{\E}{\mathbb{E}\:}
\newcommand{\PP}{\mathbb{P}}
\newcommand{\one}{\mathbf{1}}

\newcommand{\mumin}{\mu_\mathrm{min}}
\newcommand{\mumax}{\mu_\mathrm{max}}

\newcommand{\norm}[1]{\left \| #1 \right \|}
\newcommand{\lmax}{\lambda_\mathrm{max}}

\DeclareMathOperator*{\argmax}{arg max}
\DeclareMathOperator*{\argmin}{arg min}

\DeclareMathOperator{\sign}{sign}

\title{Spectral Methods for Ranking with Scarce Data}
\author{Umang Varma\footnote{Google. This work was done while at the University of Michigan.} \and Lalit Jain\footnote{University of Washington.} \and Anna C. Gilbert\footnote{Yale University. This work was done while at the University of Michigan.}}
\date{}

\begin{document}
\twocolumn
\maketitle
\begin{abstract}
  Given a number of pairwise preferences of items, a common task is to rank all the items. Examples include pairwise movie ratings, New Yorker cartoon caption contests, and many other consumer preferences tasks. What these settings have in common is two-fold: a scarcity of data (it may be costly to get comparisons for all the pairs of items) and additional feature information about the items (e.g., movie genre, director, and cast). In this paper we modify a popular and well studied method, RankCentrality for rank aggregation to account for few comparisons and that incorporates additional feature information. This method returns meaningful rankings even under scarce comparisons. Using diffusion based methods, we incorporate feature information that outperforms state-of-the-art methods in practice. We also provide improved sample complexity for RankCentrality in a variety of sampling schemes.%
\end{abstract}

\section{Introduction}
In this paper we are interested in the problem of rank aggregation from pairwise preferences under settings where the amount of data is \textit{scarce} but we may have additional \textit{structural} information.
For example, consider a setting where a set of pairwise comparisons on a set of $n$ movies have been collected from a set of critics and the goal is to give an overall ranking. If $n$ is large, for example, all movies released in the last two decades, it may be extremely costly to get a comparison for each of the $\binom{n}{2}$ pairs. A more realistic regime is to hope that each movie has been viewed at least once.
Standard methods of ranking suggest that the number of comparisons needed is roughly $O(n\log(n))$---when $n$ is large, even hoping for $\log(n)$ comparisons may be hopeless!
However, each movie has additional feature information $x_i\in \mathbb{R}^d$. For example, the dimensions could encapsulate the production budget, the number of A-list actors, the writer, studio, animated or live action, etc.
In general, we may suspect that these features inform the comparisons: if movies A and B have the same Oscar-winning director, and movie A beats movie C in a comparison, we may expect movie B to also perform well against movie C. In an extreme setting, even if we don't have any comparisons involving movie B, we may still hope to infer a meaningful ranking.
In this paper we focus on modifying a popular and well studied method arising in the ranking literature for this setting and demonstrate gains in the \textit{scarce} setting when the number of comparisons is very small.

A  common model in the literature of particular interest to us is the Bradley-Terry-Luce (BTL) model. We assume that we have $n$ items and associated to each item $i$ is a positive score $w_i$ so that the probability that $j$ is preferred to $i$ (``$j$ beats $i$'') in a comparison is
\begin{equation}
  P_{ij}:=P(i \prec j) = \frac{w_j}{w_i + w_j}\label{eq:Pij},
\end{equation}
and that we see $m$ comparisons. The underlying ranking on the items is then given by the scores $w$, with an item with a larger score being ranked higher than an item with a smaller score. In the structured setting above, we may expect movies with similar features to have similar scores. Traditional methods of learning $w$ using the BTL model, e.g., maximum likelihood estimation (MLE) or spectral methods such as Rank Centrality (both discussed below), do not naturally incorporate this kind of side information.

We have two main contributions.
\begin{enumerate}[wide, labelindent=0pt,topsep=0pt,itemsep=-1ex,partopsep=1ex,parsep=1ex]
  \item Our main contribution is Algorithm \ref{alg:regRC}, \emph{Regularized RankCentrality}, in Section~\ref{sec:regRCgeneral}. We propose a novel method for regularizing the RankCentrality algorithm that returns meaningful rankings even under scarcity. Using diffusion based methods, we propose a way of incorporating feature information that is empirically competitive with  other feature based methods such as RankSVM or Siamese Networks on both synthetic and real-world datasets in scarce settings. In a specific context, we provide a sample complexity result for this regularized method.
  \item Along the way, we discuss traditional RankCentrality and, under a natural sampling scheme extending that in \cite{agarwal14}, we show an improved sample complexity bound for the RankCentrality algorithm. For example, when pairs are sampled uniformly, we improve the bound from $O(n^5 \log n)$ to $O(n\log n)$.
\end{enumerate}

\section{Related Works}
\label{sec:relatedworks}

There is an extensive amount of literature on ranking from pairwise comparisons under various models, and we refer the interested reader to the survey in \cite{agarwal14}. Roughly speaking, most frameworks either fall into the parametric setting, i.e., a model such as BTL is assumed, or non-parametric where general assumptions on the pairwise comparison matrix $P$, where $P_{ij}$ is the probability that $i$ beats $j$ in a comparison, are made.

In the latter setting, several different conditions on $P$, such as stochastic transitivity and low noise described in \cite{agarwal14}, or low rank as in~\cite{koren2009matrix}, and generalized low permutation rank models have been proposed (see ~\cite{shah2018low}). All of these models include the BTL model as a specific case. Other estimators such as the Borda count and Condorcet winner (for finding the best item rather than a ranking) have been analyzed in ~\cite{shah2017simple}.
A variant of the ranking problem also falls under the category of active ranking where the comparisons that are queried are chosen by an active ranker rather than passively considered offline, see \cite{katariya2018adaptive, heckel2016active, jamieson2011active}. %

A great deal of attention has been paid to the BTL model. A natural approach to this setting is to compute an estimate for $w$ using the MLE. More precisely given a set of comparisons $S = \{(i_k, j_k, y_k)\}_{k=1}^m$ where the $k$-th comparison is between items $i_k$ and $j_k$, and $y_k=0$ denotes that $i_k$ was preferred in this observation, whereas $y_k=1$ denotes that $j_k$ was preferred. %
Then the MLE is given by
\begin{align}
   & \argmax_{v \in \R^n} \sum_{i=1}^m -\log \left(1 + e^{(2y_k-1)(v_{j_k} - v_{i_k})} \right) \label{eq:btlmle}
\end{align}
and our estimate is $\hat{w}_i = \exp(v_i).$ %

We can also consider a constrained MLE where we add an additional constraint\footnote{Without loss of generality, assume $\sum_i w_i = 1$ because $P_{ij}$ is invariant to scaling $w$.}, e.g., on the maximum entry of $w$, $\|w\|_{\infty} < B$, or, alternatively, we can add add an $\ell_2$ regularizer $\lambda \|v\|_2$ to the objective. The BTL-MLE in any of these formulations is a popular objective since it is convex.
We briefly review the known results on the BTL-MLE. \cite{shah2016estimation} have shown the constrained BTL-MLE is minimax optimal for the $\ell_2$ error. %
Note that low $\ell_2$ loss does not necessarily guarantee a correct recovery of a ranking.   \cite{chen2017spectral} shows that the (regularized) MLE and spectral ranking methods (discussed below) are minimax optimal for recovery of a ranking. %
The critical parameter for recovery is the minimum gap between any two different BTL scores---which does not show up when one is interested in the $\ell_2$ norm only. %

In the next section we discuss the class of algorithms that are the main study of this work: spectral methods and the RankCentrality algorithm.

\section{Spectral Methods}
\label{sec:spectralmethods}

We assume that we have access to a collection of $m$ independent and identically distributed pairwise comparisons $S = \{(i_k, j_k, y_k)\}_{k=1}^m$ where each $i_k < j_k\in [n]$. Furthermore we assume that each pair is i.i.d drawn: $(i,j) \sim_{\mu} \{(i,j), 1\leq i< j\leq n\}$, where $\mu$ is an \textit{unknown} sampling distribution on the set of ordered pairs. Although $\mu_{ij}$ is defined for $i < j$, we assume it is understood that $\mu_{ij} = \mu_{ji}$ when $i > j$. Denote $\mumin := \min_{i<j} \mu_{ij}$ and $\mumax := \max_{i<j} \mu_{ij}$.  In addition, we assume that the label is an independent Bernoulli draw, i.e.
\begin{equation*}
  y_{k} = \begin{cases}
    1 & \text{with probability } P_{i_kj_k}=\tfrac{w_{j_k}}{w_{i_k}+w_{j_k}} \\
    0 & \text{otherwise}
  \end{cases}
\end{equation*}
according to the BTL model where $(w_1, \cdots, w_n)\in \mathbb{R}_{>0}^n$ is an unknown vector of BTL-scores, i.e., $i_k\prec j_k$ with probability $P_{i_kj_k}$. Note $P_{ij} = 1- P_{ji}$. Additionally define $b := \max_{i,j} w_i/w_j$.  Without loss of generality we assume that $w^T\textbf{1} = 1$, indeed scaling the weights has no effect on the comparison probabilities.

\vspace{1ex}
\noindent \textbf{Problem.} Given $S$, return $\hat{w}$, an estimator for $w$.
\vspace{1ex}

Consider the following matrix $Q\in \mathbb{R}^{n\times n}$, defined as
\begin{equation}
  Q_{ij}  :=
  \begin{cases}
    \mu_{ij} P_{ij}                              & \text{ if } i \neq j \\
    1 - \sum_{\ell \neq i} \mu_{i\ell} P_{i\ell} & \text{ if } i = j
  \end{cases}.
  \label{eq:Q}
\end{equation}
Observe $Q_{ij}$ is the transition matrix of a time-reversible Markov chain, where the we transition from $i$ to $j$ with probability proportional to that of $i$ beating $j$ in a comparison (we refer the reader to Chapter 1 of \cite{norrismarkov} for background on Markov Chains), i.e., it  satisfies the detailed balance equations: for all $i \neq j$, we have
\begin{equation*}
  w_i Q_{ij} = \frac{\mu_{ij} w_i w_j}{w_i + w_j} = w_j Q_{ji}.
\end{equation*}
This implies the vector $w$ is the stationary distribution of $Q$, satisfying $w^TQ = w$, i.e., $w_i$ is the equilibrium probability of being in state $i$.  %
This motivates using the stationary distribution of an empirical estimator $\hat Q$, with $\E[\hat Q] = Q$ as an estimator $\hat{w}$ for $w$. The impatient reader can skip ahead to the next section for our choice of $\hat{Q}$.

The connection between the BTL model and time-reversible Markov chains was noticed by \cite{negahbanrc} where they proposed the RankCentrality algorithm for estimating $w$ under a slightly different model. In their setting, they assume they have access to a (connected) graph on $n$ vertices $G$, and for each edge in the graph they repeatedly query the associated pairwise comparison $k$ times. In the specific setting of an \Erdos--\Renyi{} graph $\mathcal{G}_{n,p}$ on $n$ vertices, %
they construct an estimator $\hat{w}$ %
and show for  $d\geq 10 C^2 \log n$ and $kd \geq 128 C^2 b^5 \log n$, setting $p = \tfrac{d}{n}$ the following bound on the error rate holds with high probability:
\[
  \frac{\big\|\hat w-w\big\|_2}{\|w\|_2} \leq  8 C b^{5/2} \sqrt{\frac{\log n}{k\,d}}.
\]
(where we recall $b := \max_{i,j} w_i/w_j$). Noting that the expected number of comparisons is $O(n^2pk) = O(nkd) =  O(b^5n\log(n))$ this yields a sample complexity of $O(b^5 n\log n/\epsilon^2)$ for recovering a weight vector with relative error $\epsilon$. Note that in this setting, for $\mathcal{G}_{n,p}$ to even be connected, it is important that $p$ be at least on order $\log(n)/n$, and we must at least observe $O(n\log(n))$ comparisons. In the more general setting, the sample complexity depends on the spectral gap of the graph Laplacian of $G$ ; precise dependencies have been given in \cite{agarwal2018accelerated, shah2016estimation}

Returning to our setting, our sampling scheme, which we refer to as \textit{independent sampling} was proposed by \cite{agarwal14}.
Observe that the independent sampling scheme is more natural in many applications, and in particular each observation is made independent of the other observations, which is not true of those in \cite{negahbanrc}.
Rajkumar and Agarwal show that if $O(\tfrac{Cn}{\varepsilon^2P_\mathrm{min}^2\mu_\mathrm{min}^2} b^3 \ln \left( \frac{n^2}{\delta}\right))$ comparisons are made %
then with probability at least $1 - \delta$ (over the random draw of $m$ samples from which $\hat P$ is constructed), the score vector $\hat w$ produced by their version of the RankCentrality algorithm satisfies $\|\hat w - w \|_2 \leq \varepsilon$.
The sample complexity here scales as $O(n^5 \log n)$ since $\mu_\mathrm{min}^{-1} \geq \binom{n}{2}$, with equality achieved only when $\mu$ is uniform. In the next section we propose a different estimator from the one given in \cite{agarwal14} and we are able to give a $O(n\log n)$ sample complexity bound in the case of uniform sampling. %

A crucial point to note is that both \cite{negahbanrc} and \cite{agarwal14} assume that the directed graph of comparisons, where an edge $(i,j)$ represents that $j$ beat $i$ in at least one comparison, is strongly connected. This is because the empirical estimate $\hat Q$ of the Markov transition matrix needs to be ergodic, i.e., irreducible and aperiodic, which ensures that $\hat Q$ has a unique stationary distribution. When the number of comparisons $m$ is small (i.e., $m<n\log(n)$ in the case of \cite{negahbanrc}), this is usually not the case and these algorithms return a default output. In particular, in the setting mentioned in the introduction where the number of comparisons are scarce, these methods will not return a useful ranking. This is a primary motivation for the work in this paper. %

\subsection{Warm-up: Improved Results for Independent Sampling}
In this section we improve the results given in \cite{agarwal14} by using a different estimator of $Q$ than the one presented there.
Recall the notation of Section~\ref{sec:spectralmethods}. Given a dataset of comparisons $S$, define
\begin{align*}
  C_{ij} = \textstyle \sum_{k=1}^m \Big( \one\{i_k = i, j_k =j, y_k=1\} \\  + \one\{i_k = j, j_k=i, y_k=0\} \Big),
\end{align*}
i.e., $C_{ij}$ is the \emph{number} of comparisons between $i$ and $j$ that $j$ won.
Additionally define the \textit{empirical Markov transition matrix}
\begin{equation}
  \hat Q_{ij}  :=
  \begin{cases}
    \frac{C_{ij}}{m}                           & \text{ if } i \neq j \\
    1 - \sum_{\ell \neq i} \frac{C_{i\ell}}{m} & \text{ if } i = j
  \end{cases}.
  \label{eq:Qhat}
\end{equation}
By construction, $Q = \E(\hat Q)$ so $\hat Q$ is an unbiased estimator of $Q$.
Let $\hat{w}$ be the leading left eigenvector of $\hat Q$. When $\hat Q$ is ergodic, $\hat{w}$ is the unique stationary distribution of $\hat Q$. %

\begin{thm}
  \label{thm:RCsamplecomp}
  Fix $\delta \in  (0, 1)$ and $\varepsilon \in (0, 1)$. If
  \[ m \geq 64b^3 n^{-1} \mumin^{-2} \varepsilon^{-2}(\mumax + n\mumax^2) \log \frac{2n}{\delta} \]
  and the empirical Markov chain $\hat{Q}$ constructed as in \eqref{eq:Qhat} is ergodic, then with probability at least $1-\delta$, we have \[\frac{\|\hat w - w\|}{\|w\|} \leq \varepsilon.\]
\end{thm}

\begin{proof}
  A complete proof can be found in the supplementary materials. We sketch an outline of the proof here.

  We first prove a result on the deviation of left eigenvectors for perturbations of ergodic row stochastic matrices, Proposition \ref{prop:Qperturbation} based on ideas from \cite{negahbanrc}. For each observation $k \in [m]$, we define a random i.i.d. matrix  $Q_k$ (in terms of $i_k$, $j_k$, and $y_k$) such that $\hat Q = I + \frac{1}{m} \sum_{k=1}^m Q_k$. We can therefore write $\hat Q - Q = \sum_k Z_k$ where each $Z_k$ is an independent random matrix with $\E(Z_k) = 0$ and we can explicitly compute the matrix variance of $Z_k$ (Lemma \ref{lem:rcvariance}). By using matrix Bernstein inequalities given in \cite{tropptailbounds} we can derive a central-limit type upper bound on $P(\|\hat w - w\| > \varepsilon)$ (Theorem \ref{thm:rcconvergence}). Solving the resulting inequality for $m$, we get the desired result.
\end{proof}

Because $\mumin = \mumax = \binom{n}{2}^{-1}$ when $\mu$ is uniform, we have given an $O\left(b^3\varepsilon^{-2}n\log(\tfrac{n}{\delta})\right)$ sample complexity when $\mu$ is uniform. Our argument improves upon that in \cite{agarwal14} through improved matrix concentration results and a different (unbiased) estimator for $Q$.%

\section{Regularizing RankCentrality}
\label{sec:regRCgeneral}

When the number of pairwise comparison observations we have available is small, the $\hat Q_{ij}$ entries are poor estimators for $Q_{ij}$: there are $n^2 - n$ off-diagonal entries in $\hat Q$ and each observation only affects one off-diagonal entry leaving most entries zero. Furthermore, as described in the previous section, if the graph of pairwise comparisons (given by connecting any two points with an edge) is not strongly connected, may not guarantee that $\hat{Q}$ has a unique stationary distribution.
\textbf{Motivated by this, we ask a natural question---when the number of pairwise comparisons is small; i.e., data is scarce (for example we have just observed one comparison per item) how can we still obtain a reasonable ranking?}

Intuitively, if the items $[n]$ have some inherent structure, we can hope to exploit that structure to infer pairwise comparisons. %
Since $Q_{ij} = \mu_{ij}P_{ij}$; i.e., a scaled probability of $i$ beating $j$, even if we have never seen a comparison between $i$ and $j$, it is reasonable to estimate this value by taking a weighted combination of the empirical $\hat Q_{ik}, 1\leq k\leq n$, where the choice of weights perhaps reflect some prior knowledge on the similarity between $j$ and $k$. In an extreme case---if we suspect item $j$ and $k$ would perform the same against item $i$, we may choose the weight on $\hat Q_{ik}$ to be large, and set the weights on all other $\hat Q_{ik'}, k\neq k'$ to zero.

Said more precisely, we choose a row-stochastic matrix $D$ and use the estimator $\hat{Q}D$ whose $ij$-th entry is \begin{equation} \label{eq:QDdefn}
  [\hat Q D]_{ij}  = \sum_{k=1}^n D_{kj} \hat  Q_{ik}
\end{equation}
How should we choose $D$? We want $\hat QD$ to be ergodic, but it should also reflect some similarity structure between the items. This prior information could take form in many ways---for example  we can imagine that associated to item $i$ is a feature vector $x_i\in \mathbb{R}^d$ and intuitively items that are close together perform similarly on a comparison with some other element $j$ (see Section~\ref{sec:similarity}). An extreme case of this is assuming that the items are in clusters, and items within a cluster rank similarly (or the same). Finally, we can consider forms of $D$ that do not reflect any prior structure but do at least guarantee that $\hat QD$ is ergodic---as we will show these estimators can still perform competitively with other methods (Section~\ref{sec:regRC}). To recap, our resulting regularized RankCentrality algorithm that we will discuss in the rest of this section is given below in Algorithm \ref{alg:regRC}.
\begin{algorithm}
  \caption{Regularized RankCentrality algorithm}\label{alg:regRC}
  \begin{algorithmic}[1]
    \Procedure{RankCentrality}{$n,S, D$}
    \State \textbf{compute} $\hat Q$ as in \eqref{eq:Qhat}
    \State \textbf{return} leading left eigenvector of $\hat Q D$
    \EndProcedure
  \end{algorithmic}
\end{algorithm}

\subsection{Diffusion Based Regularization}
\label{sec:similarity}

Diffusion RankCentrality leverages additional features $x_i \in \R^d$ for each of the items $i \in [n]$ being ranked. We use this to compute pairwise similarities in a manner consistent with the literature (e.g., in $t$-SNE \cite{maaten2008visualizing} and diffusion maps formulated by \cite{Coifman2005diffusionmaps}) so that for a fixed $i$, the similarities $D_{ik}$ are proportional to the probability density of a Gaussian centered at $x_i$. Let $D^{(\sigma)}_{ik}$, the similarity between item $i$ and $j$, be defined as
\begin{equation}
  D_{ik}^{(\sigma)} := \frac{\exp\left(\frac{-\|x_i -x_k\|^2}{\sigma^2}\right)}{\sum_{l=1}^n \exp\left(\frac{-\|x_i -x_l\|^2}{\sigma^2}\right)},
  \label{eq:similarityD}
\end{equation}
where $\sigma$, the kernel width, is an appropriately chosen hyperparameter.  The Diffusion RankCentrality algorithm, obtained by using $D^{(\sigma)}$ in Algorithm \ref{alg:regRC}, returns the stationary distribution of the Markov chain $\hat Q D^{(\sigma)}$.

As described in equation~\eqref{eq:QDdefn}, $[\hat Q D^{(\sigma)}]_{ij} = \sum_{k=1}^n D_{kj}^{(\sigma)} \hat Q_{ik}$,
i.e., the $ij$ entry is a weighted average of $\hat Q_{ik}$'s. $D_{ij}^{(\sigma)}$ is large when $x_i$ is close to $x_j$  and close to 0 when they are far apart. In particular the $\hat{Q}_{jk}$ contribute more when $j$ is close to $i$ and less otherwise.%

An alternative interpretation of this procedure is given by considering the Markov chain induced by $\hat{Q}$ and contrasting it with that of $\hat{Q}D^{(\sigma)}$. Consider starting at any item $i$, and repeatedly transitioning according to $\hat{Q}$. If the number of comparisons is small, there may not even be a path from $i$ to any other item $j$. In addition, any additional comparison greatly affects the stationary distribution (i.e. the limiting distribution as we transition according to $\hat{Q}$) of $\hat{Q}$. Contrast this with the stationary distribution of $\hat{Q}D^{(\sigma)}$. By construction, $\hat{Q}D^{(\sigma)}$ will be dense (assuming each element has some neighbor that has a comparison). We can interpret the elements of $\hat{Q}D^{(\sigma)}$ as a Markov chain themselves: first, we make a sub-step (say from $i$ to $k$) according to $\hat Q$, which is based only the pairwise comparison observations, and then we make a sub-step (say from $k$ to $j$) with probability that inversely depends the distance of points to $k$. In, particular, we have imputed a series of transitions from $i$ to other elements $j$, using the underlying geometry of the points along with the pairwise comparisons. This technique is similar to that found in~\cite{magic}, the MAGIC algorithm used in the field of single-cell RNA sequencing, where each entry in $Q$ is an extremely undersampled low integer count.

\vspace{1em}
\begin{ex}
  \hrule
  Consider the following extreme case example. Suppose the 100 points $\{x_i\}_{i=0}^{99}$ lie in 10 tight clusters with cluster $k$ being $\{x_{10k+1}, \cdots, x_{10k+9}\}$ and the clusters are spaced very far apart. Assume the BTL scores of items are constant within clusters; if items $i$ and $j$ are in the same cluster  then $x_i = x_j$ and $w_i = w_j$. Set $\|x_i - x_j\| = \infty$  when $i$ and $j$ are in different clusters.
  In this case, the matrix $D^{(\sigma)}$ is block diagonal: $D^{(\sigma)}_{ij} = \frac{1}{10}$ when $i$ and $j$ are in the same cluster and $D^{(\sigma)}_{ij} = 0$ otherwise.

  Figure \ref{fig:Qheatmaps} demonstrates the benefit of multiplying $\hat Q$ by $D^{(\sigma)}$. We see that a comparison between $i$ and $j$ does not just affect the $ij$ entry, but those corresponding to neighbors of $i$ and $j$. To visualize the effect of $D^{(\sigma)}$, we also show heatmaps of the 50-th powers of the transition matrices, $\hat Q$ and $\hat Q D^{(\sigma)}$. The checkered patterns in $Q$ and $QD^{(\sigma)}$ are clearly visible in $(\hat QD^{(\sigma)})^{50}$ while $\hat Q^{50}$ is still very sparse. After 50 iterations of $\hat{Q}$ vs. $\hat{Q}D^{(\sigma)}$, we see the impact of regularization, $(\hat QD^{(\sigma)})^{50}$ is far less sparse than $\hat{Q}^{50}$ and reflects a block structure that is imputing comparisons for items that have been compared less often.

\end{ex}
\hrule

\begin{figure}
  \begin{subfigure}[t]{0.48\linewidth}
    \centering
    \includegraphics[width=\linewidth]{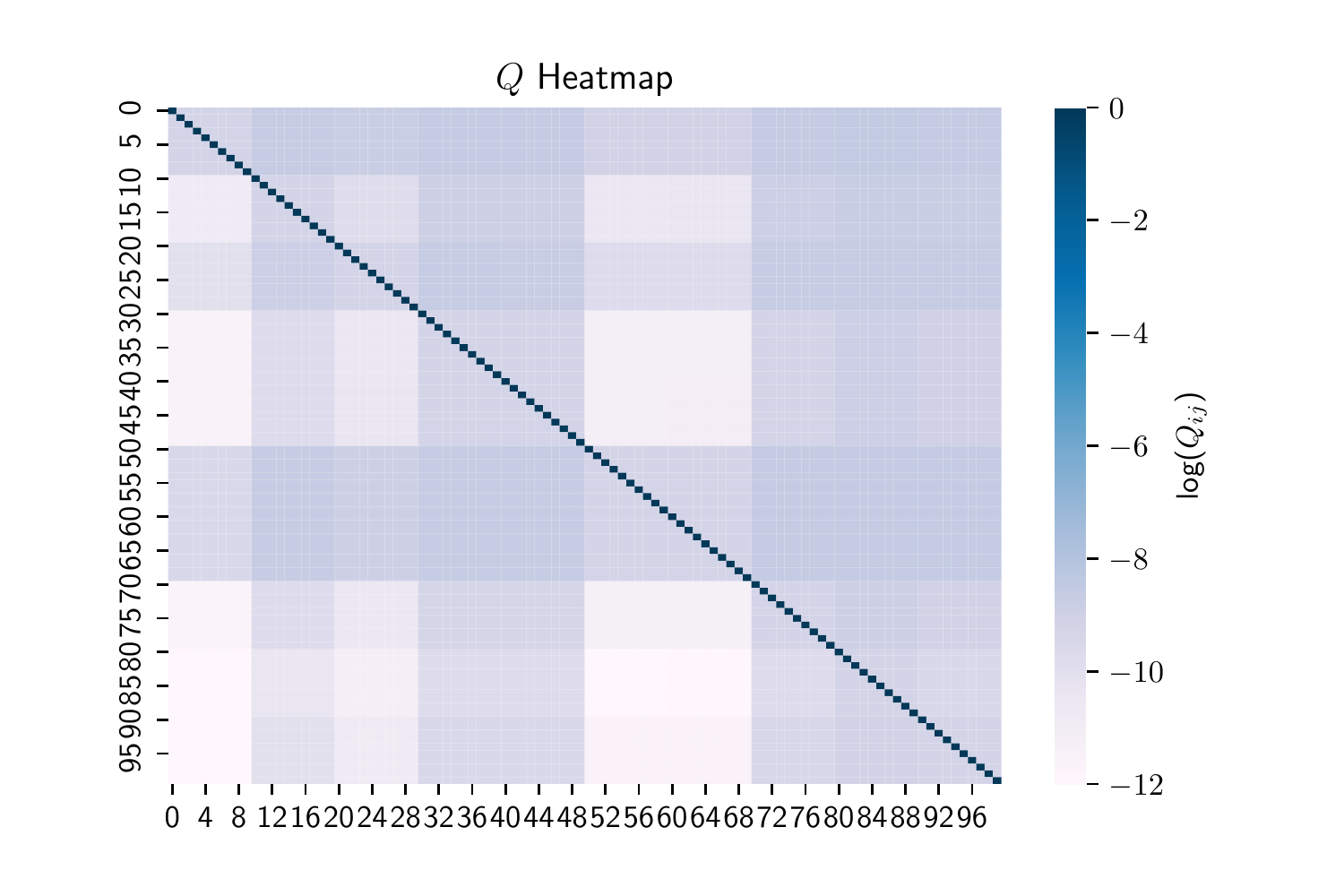}
  \end{subfigure}
  \begin{subfigure}[t]{0.48\linewidth}
    \centering
    \includegraphics[width=\linewidth]{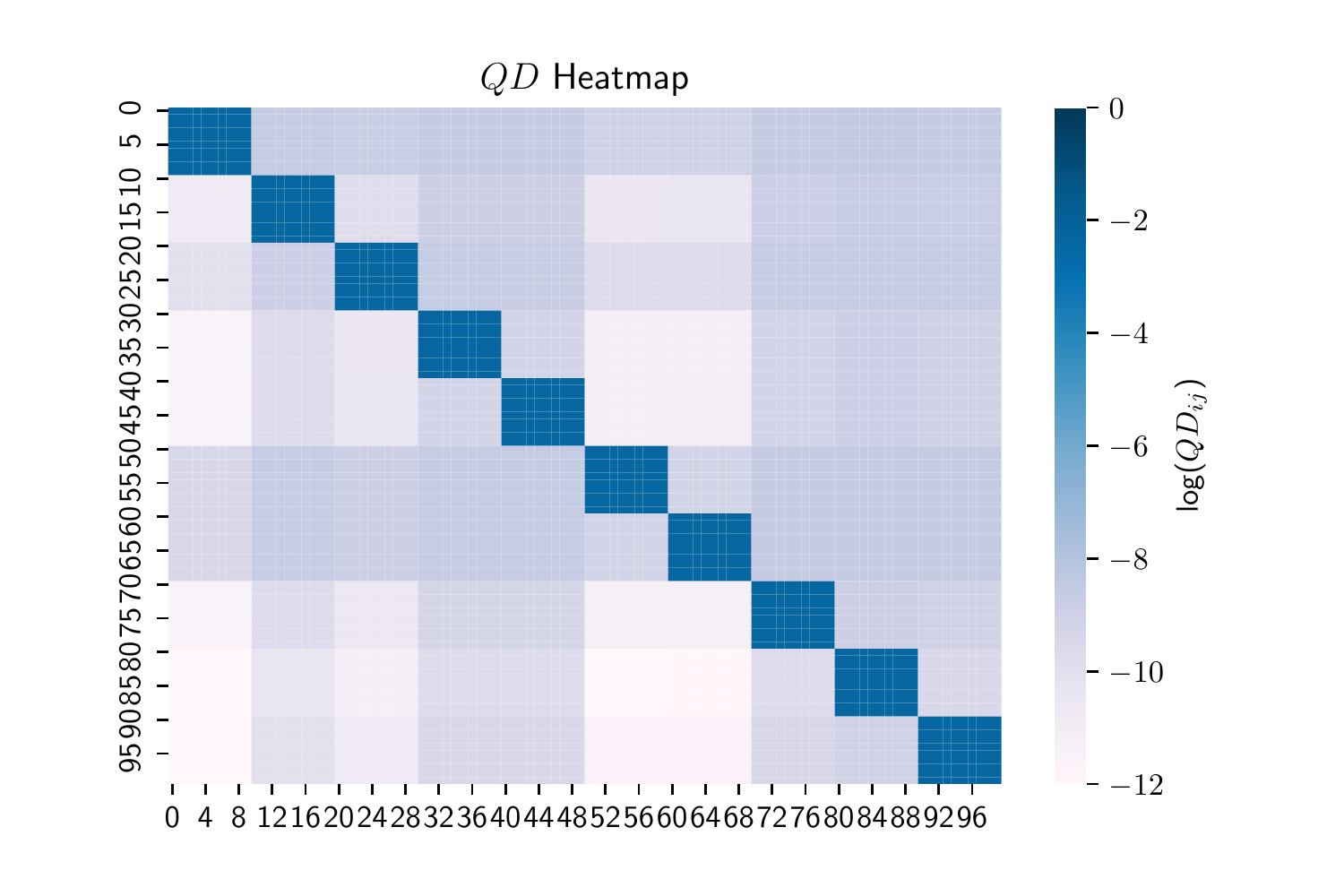}
  \end{subfigure}

  \begin{subfigure}[t]{0.48\linewidth}
    \centering
    \includegraphics[width=\linewidth]{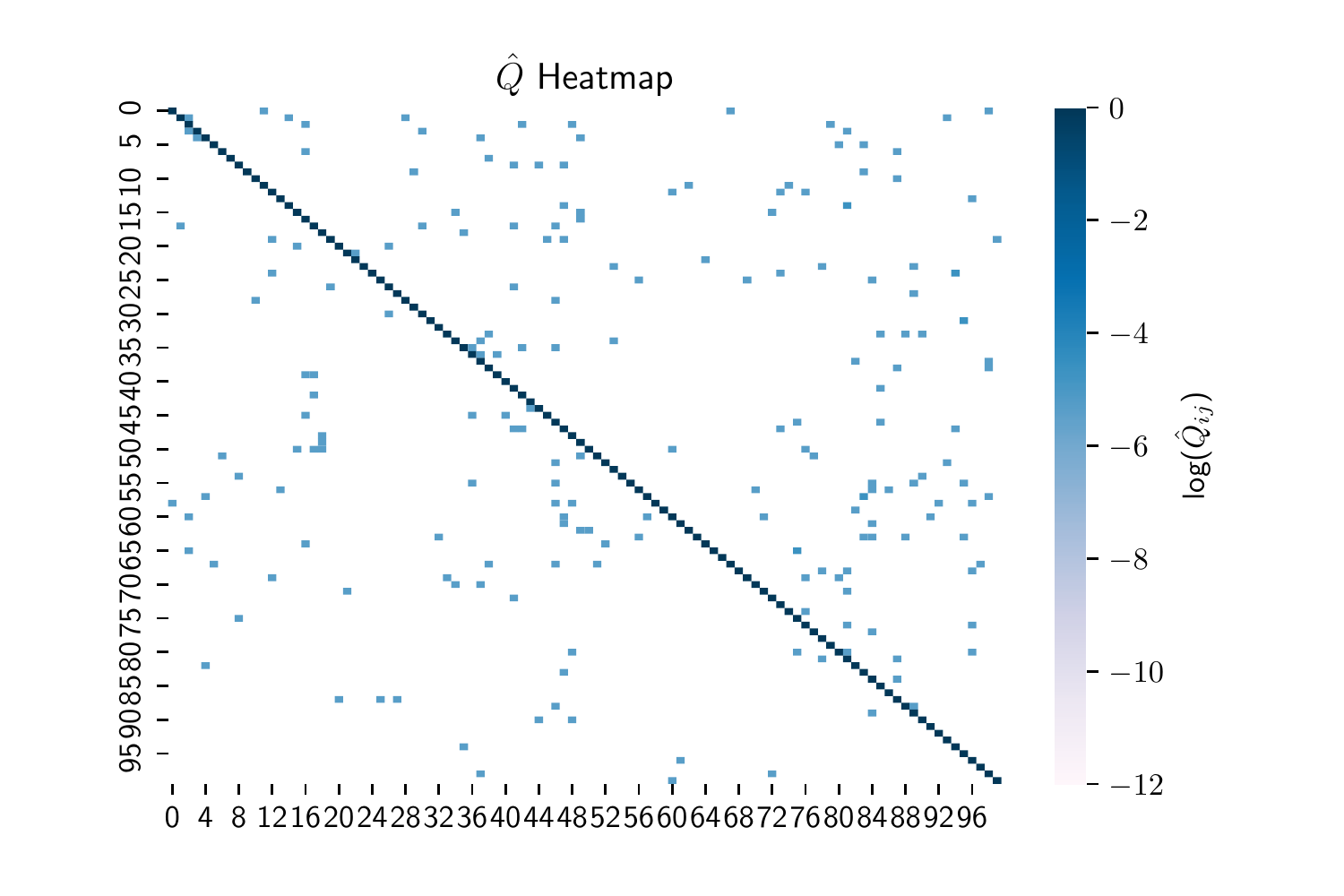}
  \end{subfigure}
  \begin{subfigure}[t]{0.48\linewidth}
    \centering
    \includegraphics[width=\linewidth]{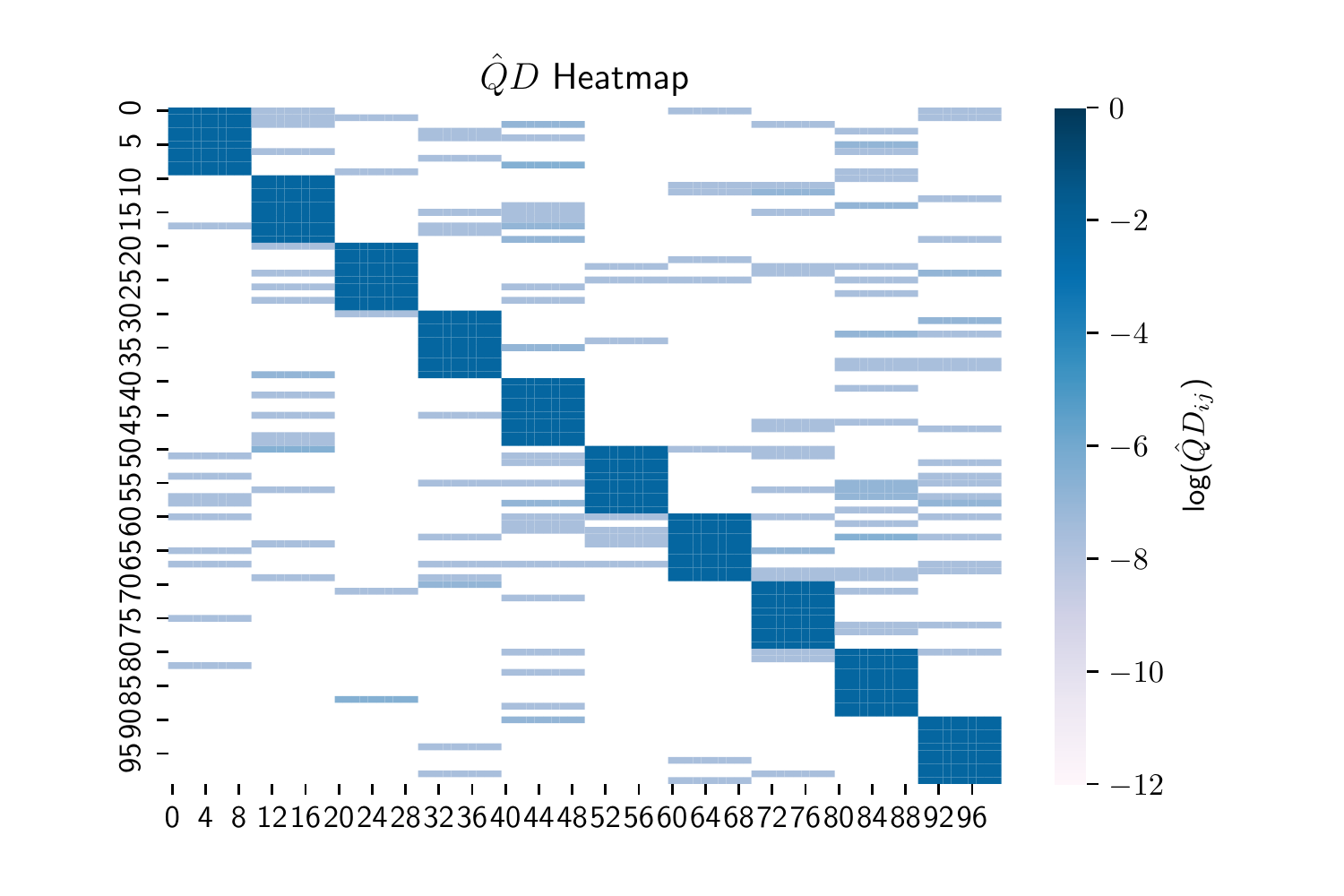}
  \end{subfigure}

  \begin{subfigure}[t]{0.48\linewidth}
    \centering
    \includegraphics[width=\linewidth]{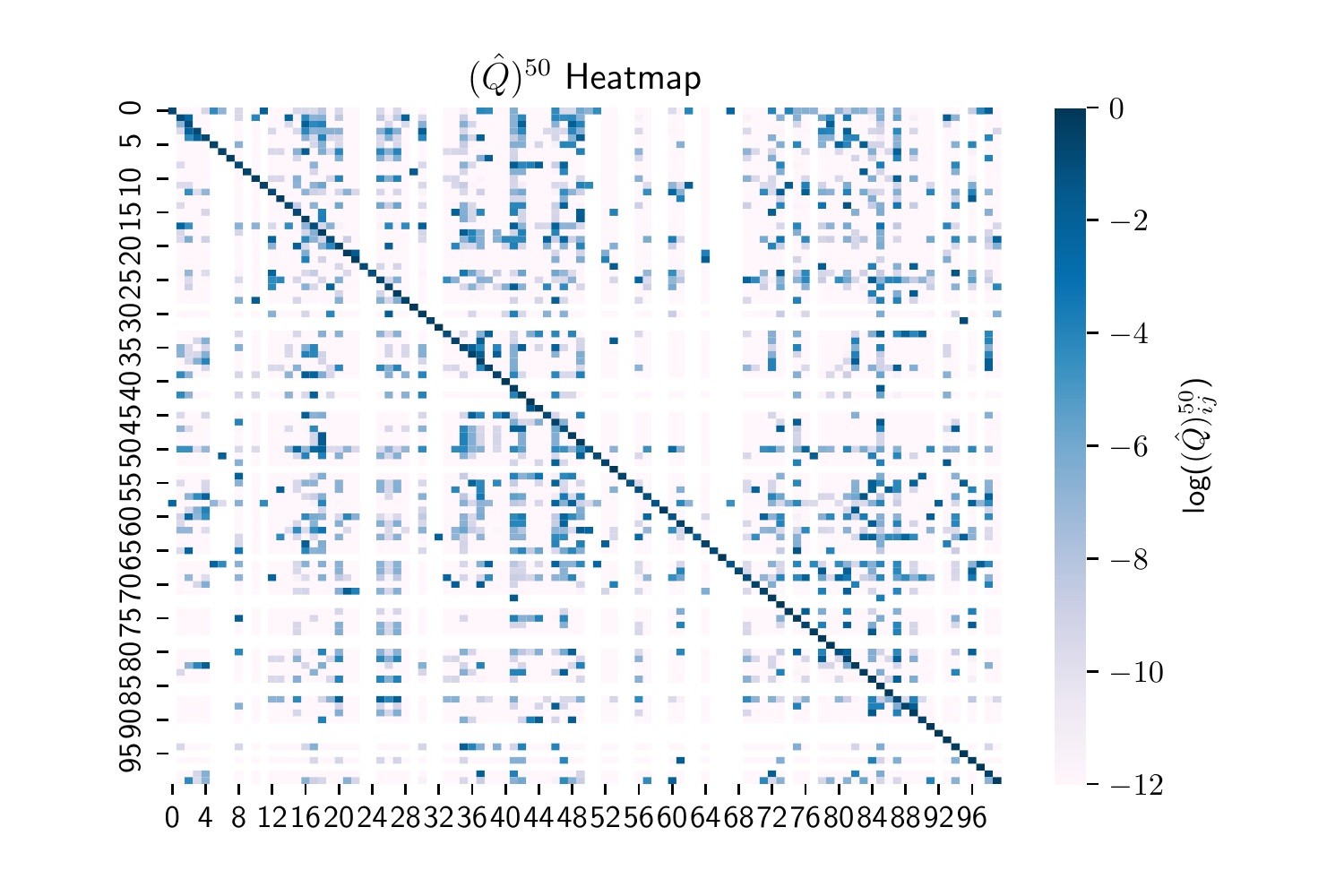}
  \end{subfigure}
  \begin{subfigure}[t]{0.48\linewidth}
    \centering
    \includegraphics[width=\linewidth]{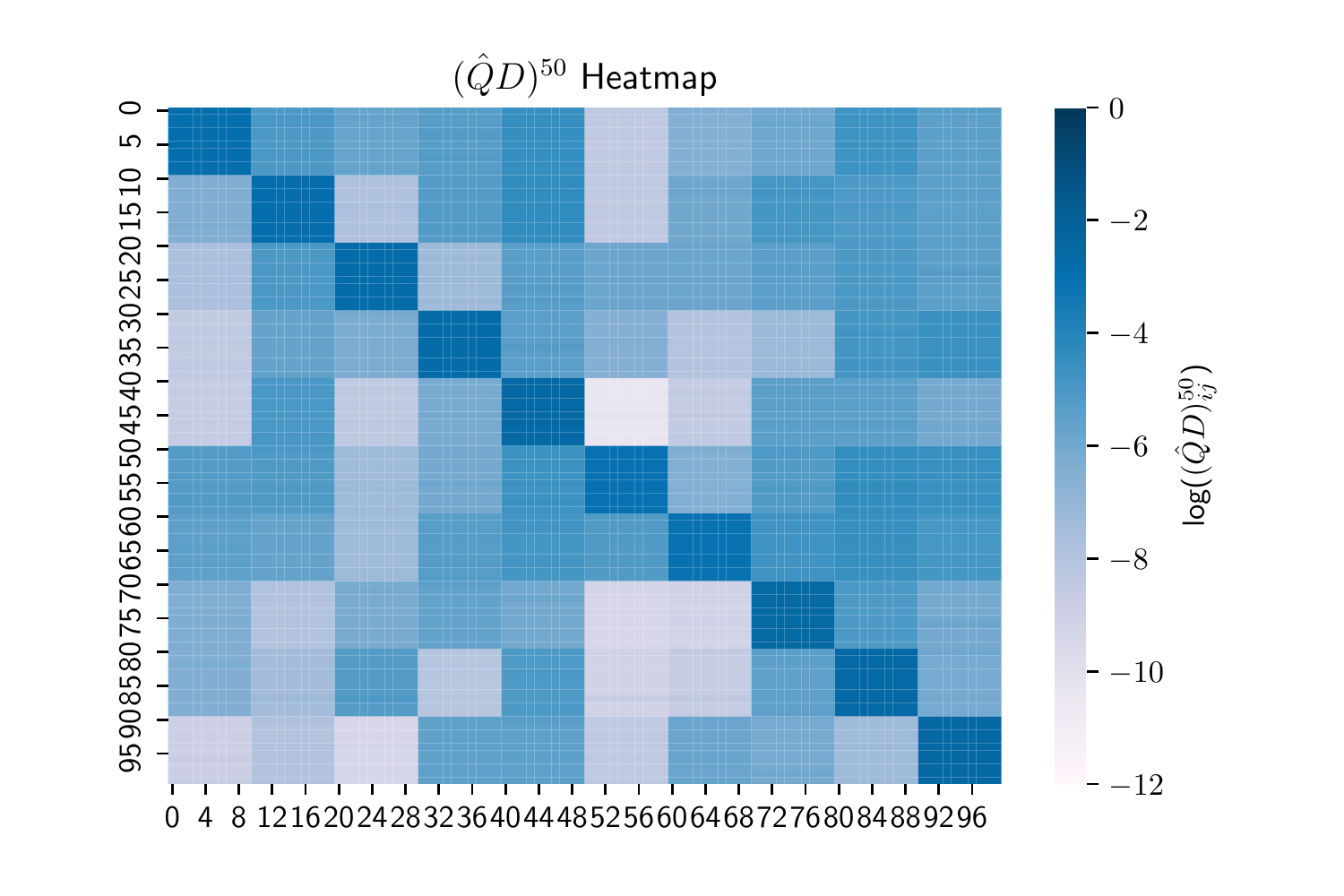}
  \end{subfigure}
  \caption{Demonstrating the impact of $D^{(\sigma)}$.
    The 100 items in this experiment lie in 10 equally sized tight clusters, where BTL scores are constant within clusters and the corresponding $D^{(\sigma)}$ matrix is block diagonal. The $\hat Q$ matrix was computed using 200 pairwise comparisons simulated according to the BTL model. \label{fig:Qheatmaps}}
\end{figure}

There are a number of different ways we could have diffused the information across the samples. We could have used $\hat Q D^{(\sigma)}$, $D^{(\sigma)}\hat Q$, or even $D^{(\sigma)} \hat Q D^{(\sigma)}$. In our empirical analysis, however, we found no significant difference in the performance of the algorithm run with these possibilities.%

Finally, we note that the running time of the regularized RankCentrality algorithm is dominated by the computation of the leading eigenvector. The matrices $Q$ and $D$ are of size $n \times n$ and we can form the matrix $M = \hat Q D$ in time $O(n^3)$. We then iterate in the power method with $M$, each iteration, requiring a matrix-vector multiply takes time $O(n^2)$. %
Our empirical analysis suggests that a few steps of the power method are sufficient. Furthermore, this iterative eigenvector computation on sparse matrices can be faster, than optimization procedures inherent in the MLE.

\subsection{\texorpdfstring{$\lambda$}{Lambda}-Regularized RankCentrality}
\label{sec:regRC}

Implicitly, $D$ is chosen so that two properties are satisfied. Firstly, $\hat QD$ will be an ergodic markov chain, and secondly, as in most regularization situations, we choose $D$ to capture some inherent prior structural information we may have about $w$ apriori. In this section we ignore the second motivation and instead focus on a $D$ which just guarantees that former constraint.

In particular, given $\lambda > 0$ we consider $D_\lambda := (1-\lambda) I + \frac{\lambda}{n} \one \one^T$ as a choice of regularizer in Algorithm \ref{alg:regRC}. Note that $\hat QD_\lambda = (1-\lambda) \hat Q + \frac{\lambda}{n} \one \one^T$, which ensures that $\hat Q D_\lambda$ is a positive row-stochastic matrix, \textit{which must be ergodic}.
In particular, we can run Algorithm 1, regardless of the number of samples and we are guaranteed that $\hat Q D_\lambda$ necessarily has a unique stationary distribution. The simple nature of $D_{\lambda}$ allows us to give a precise theoretical characterization of it's performance.  %
In general, $\E[\hat QD_{\lambda}] = QD_{\lambda}$, but $QD_{\lambda}$ may not have the same left eigenvector as $Q$. This introduces a bias in our estimator.
How can we overcome this bias? Inspecting the form of $D_{\lambda}$, note that if $\lambda \to 0$ as $m\rightarrow \infty$ then $D_{\lambda} \rightarrow I$. The following theorem characterizes the error of this procedure of any $\lambda$ and shows that it is reasonable to take $\lambda = O(1/\sqrt{m})$.
For notational convenience, we let $\gamma := \frac{n\mumin}{2(1+\sqrt{2})b^{3/2}}$.
Note that $\gamma$ is not constant---in fact it is
$O(\frac{1}{n})$.

\begin{thm}
  \label{thm:regRCsamplecomp}
  Let $\lambda \in (0, \frac{\gamma}{2})$. Choose $\delta \in  (0, 1)$ and  $\varepsilon  \in \left(2\lambda\gamma^{-1}, 1\right)$. Let $\hat w_{\lambda}$ be the output of Regularized RankCentrality run with $D=D_{\lambda}$. Then, with probability at least $1-\delta$,
  \begin{equation*}
    \frac{\|\hat w_{\lambda} - w\|}{\|w\|} < \ 2 \lambda \gamma^{-1}  +\sqrt{\tfrac{68(1-\lambda)b^{3}(\mumax + n\mumax^2)}{n\mumin^2 m} \log\frac{2n}{\delta}},
  \end{equation*}
  In particular, choosing $\lambda = c/\sqrt{m}$, then with probability at least $1-\delta$, we have
  \[ \frac{\|\hat{w} - w\|}{\|w\|} = O\left(\frac{b^3\log(2n/\delta)}{n\mu_{\min}m}\right). \]
\end{thm}
We give a proof in the supplementary material under Corollary \ref{cor:regRCsamplecomp}.

Our empirical experiments run with $\lambda = \eta m^{-1/2}$ for various values of $\eta$ support decaying $\lambda$ in this way. Figure \ref{fig:regrcwii} demonstrates a run of  $\lambda$-Regularized RankCentrality on a setting where $w = [i]_{i=1}^{200}$ and the underlying distribution on pairwise comparisons is assumed to be uniform. We compare several choices of $\lambda$ (with $\lambda = 0$ corresponding to normal RankCentrality) and the BTL MLE with an $\ell_2$ regularizer\footnote{Without such a regularizer, the BTL-MLE is underdetermined when the number of comparisons is small and cannot be solved.} on the weights (implemented using logistic regression).  Note that $\eta = 1/6$ seems to perform the best and even outperforms regularizing the BTL-MLE for small sample sizes where RankCentrality may still be returning a uniform distribution.
For more details and experiments with different choices of $w$ in this setting, see Appendix \ref{sec:lambdaregempirical} in the supplementary materials.

\textbf{Remark:} To connect the diffusion based regularization with $\lambda$-regularization, observe that if we take $\sigma \to 0$ in the definition of $D$ in Equation~\ref{eq:similarityD}, then $D\to D_0 = I_n$ (when the $x_i$'s are all distinct). The kernel width $\sigma$, therefore, determines the bias of Diffusion RankCentrality---small values of $\sigma$ only introduce a small bias in the algorithm while large values of $\sigma$ introduce considerable bias. Motivated by Theorem \ref{thm:regRCsamplecomp}, to diminish this bias as $m$ increases, we can use $(1-\tfrac{1}{\sqrt{m}})I + \frac{1}{\sqrt{m}}D^{(\sigma)}$ in Diffusion RankCentrality instead of $D^{(\sigma)}$ directly. We call this \emph{Decayed Diffusion RankCentrality}. In general, cross-validation could be used to choose the kernel width.

\begin{figure}
  \centering
  \includegraphics[width=\linewidth]{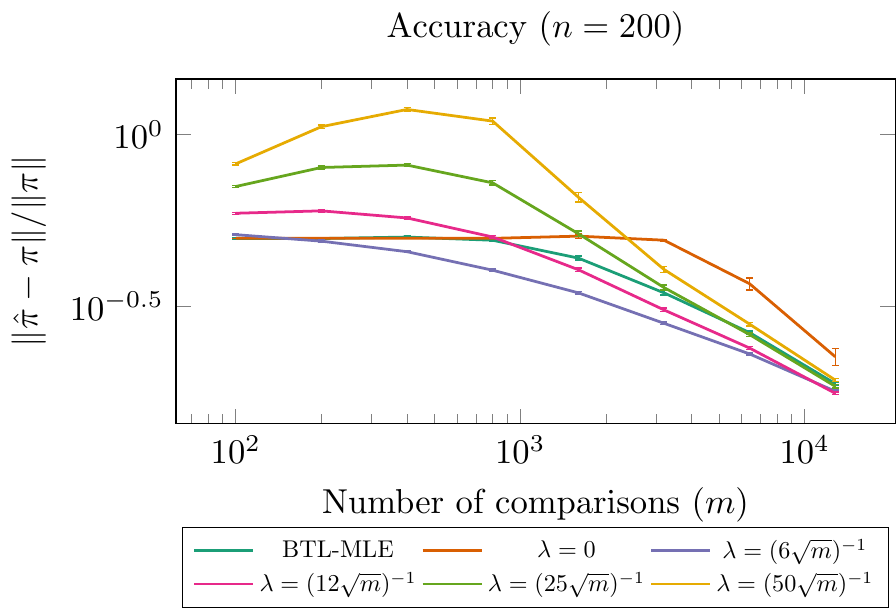}
  \caption{Comparing $\lambda$-Regularized RankCentrality with BTL-MLE and RankCentrality. Here $w = [i]_{i=1}^{200}$.\label{fig:regrcwii}}
\end{figure}

\section{Empirical Results for Regularized RankCentrality}
\label{sec:empirical}

In this section we do a comparison of the regularized RankCentrality methods in the structured setting to standard methods for ranking on synthetic and real world datasets. The code we used along with additional plots are part of the supplementary material. Although our theoretical analyses do not make assumptions about $\mu$, our experiments focus on the case where $\mu$ is uniform.

\subsection{Comparison to Scoring Functions}
As discussed in Section \ref{sec:relatedworks}, there is a rich literature of ranking methods, though less so for ranking data that come with features. Recall, we assume for each item $i \in [n]$ there is a vector $x_i \in \R^d$. In past work, the goal is to learn a function $f:\mathbb{R}^d\rightarrow\mathbb{R}$, presumed to be in a specified function class $\mathcal{F}$, such that $\sign(f(x_i) - f(x_j))$ predicts a comparison between item $i$ and item $j$. To learn $f$ given the dataset $S = \{(i_k, j_k, y_k)\}_{k=1}^m$, and a loss function $\ell:\mathbb{R}\times\mathbb{\R}\times \{0,1\}\rightarrow\mathbb{R}$, we can learn the empirical risk minimizer $\argmin_{f\in \mathcal{F}}\sum_{k=1}^n \ell(f(x_i), f(x_j), y_k)$. Two notable examples that focus on learning a scoring function that we compare to are RankSVM by~\cite{ranksvm} and Siamese network based approaches due to~\cite{siamese}.

RankSVM assumes that $\mathcal{F} = \{f: x \mapsto w^T x\}$, i.e. linear separators through the origin and choose $\ell(f(x_i), f(x_j), y) = \min(0, 1-(f(x_i)- f(x_j))(2y-1)$.
When testing RankSVM, we used it naively on the original features but also considered a kernelized version using random features, as described in \cite{randomfeatures} and implemented in SkLearn, \cite{scikit-learn}. %

Note that when the loss function is the logistic loss, $\ell(f(x_i), f(x_j), y) = \log\left(\frac{\exp(f(x_j))}{\exp(f(x_i))+\exp(f(x_j))}\right)$,
we recover the MLE under the assumption that the BTL scores are given by a transformation of the features. Such an objective has been proposed several times in the literature, e.g. \cite{burges2005learning}.
In the extreme case $f(x_i) = \theta_i$ is the BTL-MLE.

An example of such an approach are Siamese Nets, introduced by in \cite{siamese}.
We implemented a Siamese network using Keras (\cite{keras}) with two hidden dense layers, each with 20 nodes and a dropout factor of 0.1, and an output dimension of 1. Each layer in the base network used a ReLU activation. The outputs of the right network is subtracted from that of the left and a cross-entropy loss is then used.

We point out that in general both methods described above have a very different goal from what our paper proposes. Our goal is \textbf{not} to learn a scoring function, but instead to use the similarity information to inform the ranking process. In general, learning a scoring function can be expensive in terms of both computation, and samples. In addition, if the features do not actually inform the ranking very well, we want methods that will still learn a reasonable ranking---guaranteed by regularized RankCentrality as $m\rightarrow \infty$. We now demonstrate competitive performance of regularized RankCentrality even when the data is generated by a scoring function.

We constructed two synthetic datasets. We assume that the BTL-score is given by a continuous function of the features; i.e., there is an $f:\mathbb{R}^d\rightarrow \mathbb{R}$ so that the BTL score $w_i = f(x_i)$. This intuitively captures the idea that items which are close in space are close in rank. We consider a few examples of such functions $f$ as given below. %

\begin{itemize}[leftmargin=*]
  \item In Experiment A, we generated 1600 points $\{x_i\}_{i=1}^{1600}$ chosen uniformly at random from $[0, 4]^2$, we chose $\omega_1, \omega_2, \dots, \omega_4 \in \R^2$ at random, each entry chosen independently from a Gaussian. To each $i \in [1600]$ we associate a score $w_i = \sum_{h=1}^2 \exp(\cos(5 \omega_h^T x_i)) + \sum_{h=3}^4 \exp(\omega_h^T x_i / 10)$.
  \item In Experiment B, we generated 1000 points $\{x_i\}_{i=1}^{1000}\in [0,4]$ chosen uniformly at random and chose $\omega \in \R$ at random from a Gaussian. To each $i \in [1000]$ we associate a score $w_i = \exp(\cos(5\omega x_i))$.
\end{itemize}
For varying of $m$, we simulated $m$ observations under the BTL-model with uniform $\mu$ and ran various algorithms that have been discussed. We recorded plotted the average Kendal-tau correlation metric (see Section \ref{sec:kendalltau} in the supplementary for details) between the ranking on the synthetic scores we generated and the true ranking on the items. The results of these experiments are summarized in Figures \ref{fig:exp2} and \ref{fig:exp1}.

\begin{figure}[htb]
  \centering
  \includegraphics[width=\linewidth]{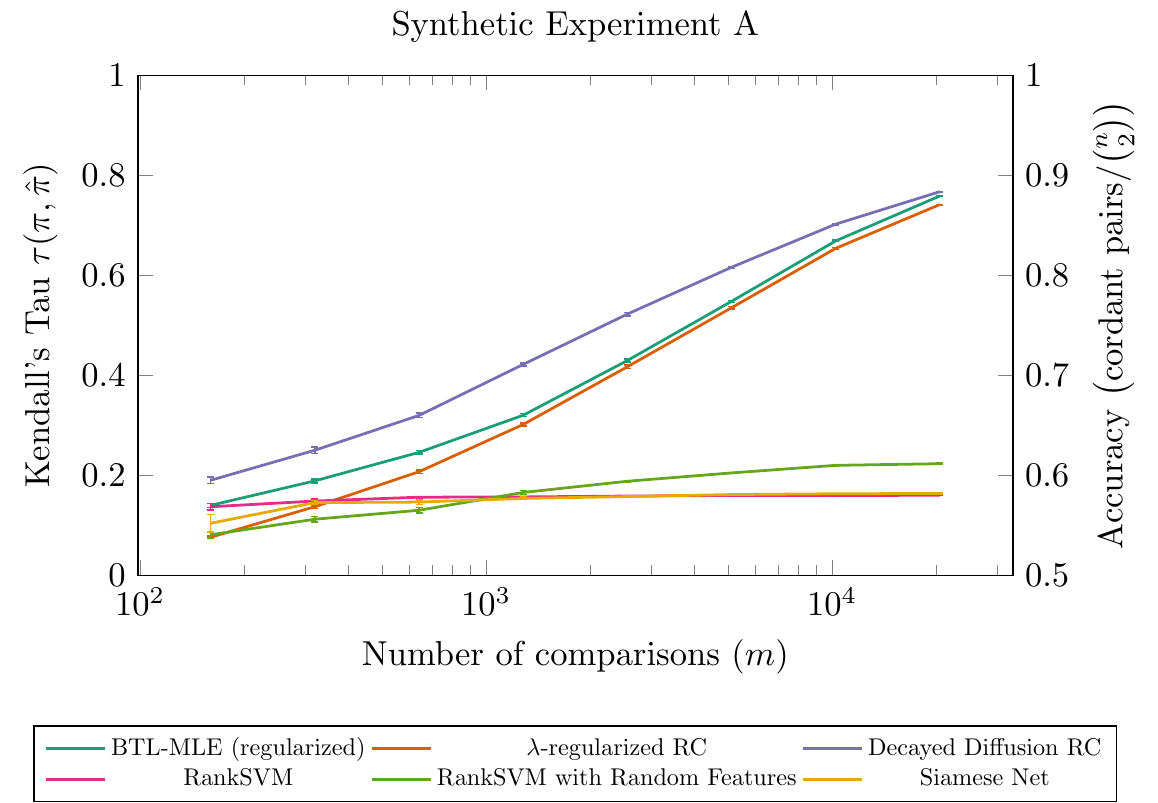}
  \caption{Comparison of algorithms in synthetic experiment A. Diffusion RankCentrality was run with kernel width $\sigma = 2^{-4}$. \label{fig:exp2}}
\end{figure}

\begin{figure}[htb]
  \centering
  \includegraphics[width=\linewidth]{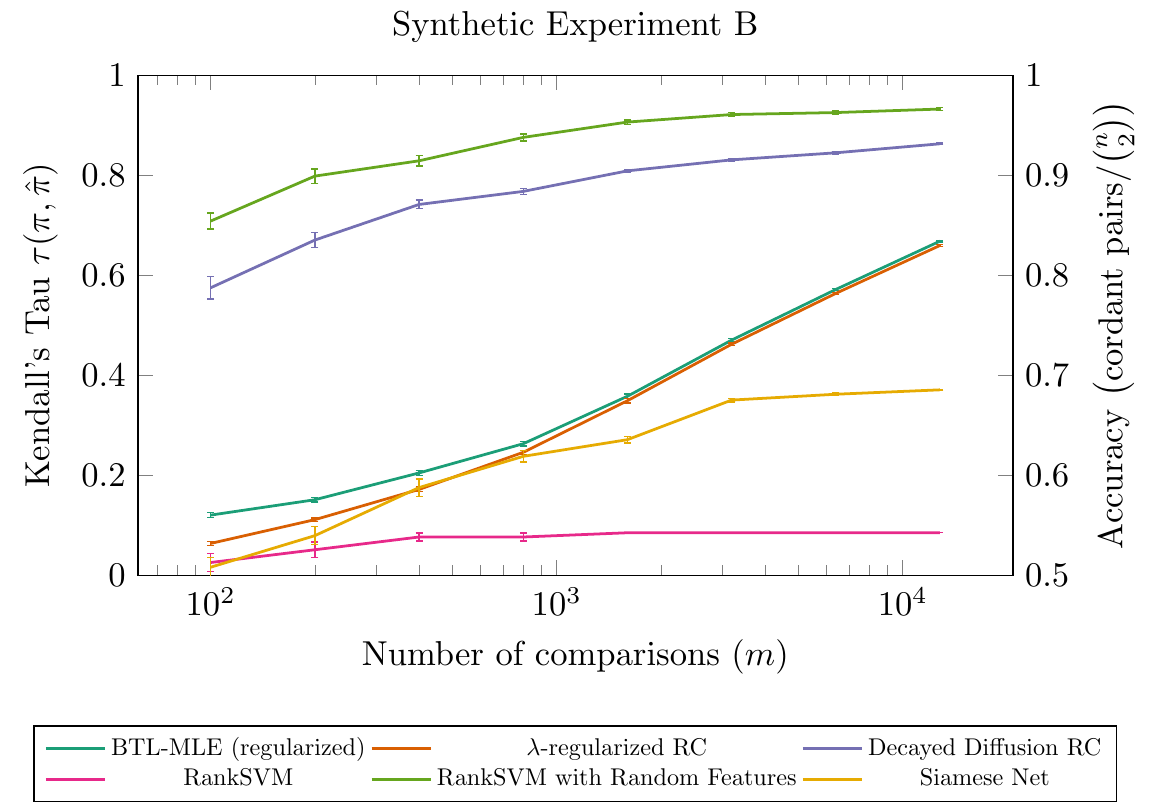}
  \caption{Comparison of algorithms in synthetic experiment B. Diffusion RankCentrality was run with kernel width $\sigma = 2^{-5}$.  \label{fig:exp1}}
\end{figure}

In Experiment A, Diffusion RankCentrality proves to be the best method when the comparisons are scarce. The impact of Diffusion RankCentrality in Experiment B is dramatic when compared to $\lambda$-regularized RankCentrality. While it is true that RankSVM with random features far outperforms other algorithms, it should not come as a surprise given that the BTL scores $w_i$, as a function of $x_i$, come from monotonic transformations of linear combinations of the basis of the RKHS used for the implementation of random Fourier Features in scikit-learn \cite{scikit-learn}.

In both experiments, Diffusion RankCentrality outperforms Siamese Networks.
To choose the kernel width, we ran Decayed Diffusion RankCentrality with several different choices of $\sigma$ on a validation set and chose the best one (see Figure \ref{fig:exp1widths}).

\begin{figure}
  \centering
  \includegraphics[width=\linewidth]{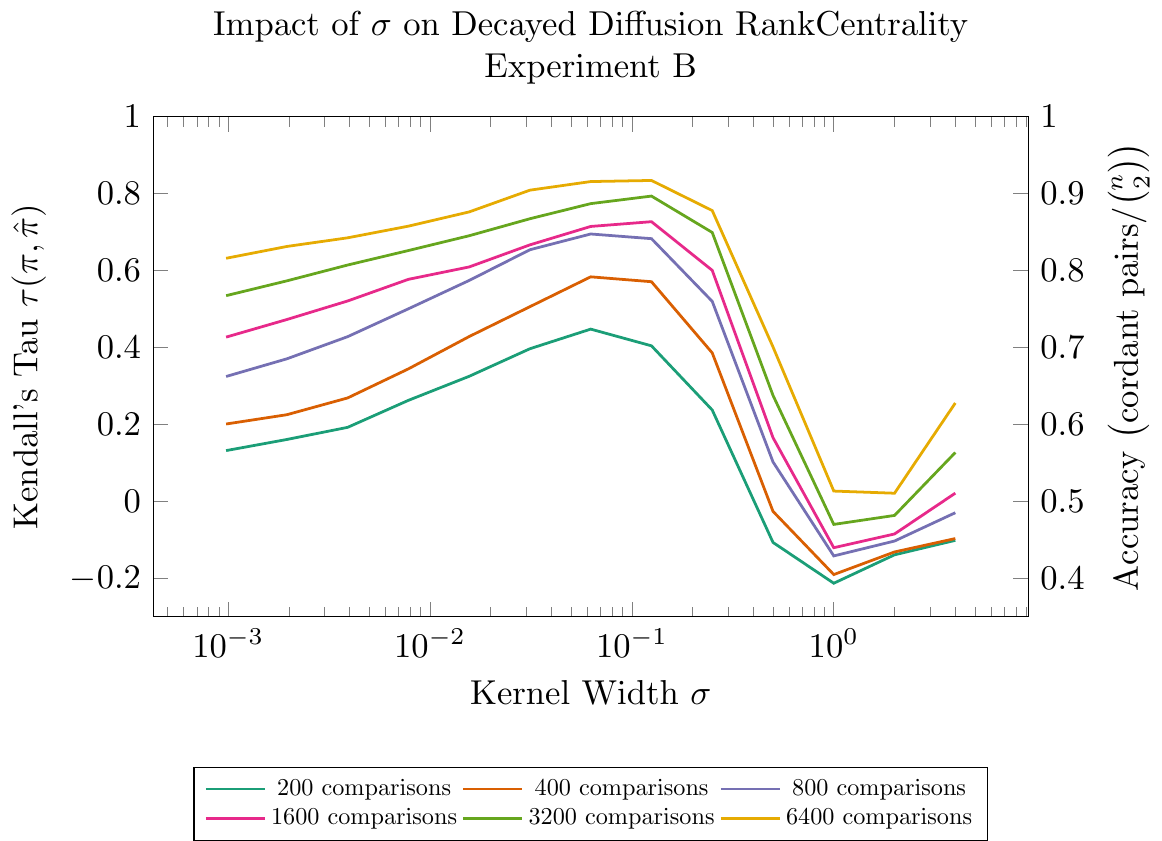}
  \caption{Impact of kernel width on performance of Diffusion RankCentrality. \label{fig:exp1widths}}
\end{figure}

\subsection{New Yorker Caption Competition}
\label{sec:ny651}

It is challenging to find real-life data sets that satisfy all of the following conditions:  1) The data is structured; i.e., has image or text features associated with the items and 2) the number of items compared is moderate to large in size.

The New Yorker Caption Competition dataset consists of a cartoon and a series of associated (supposedly) funny captions submitted by readers (see \cite{nextdata} for details on this dataset). Each week, readers vote on whether they think each caption is funny (2 points), somewhat funny(1 point) or unfunny (0 points), and the caption is assigned an average cardinal score based on these points. Included in this dataset are only two contests (\#508 and \#509), in which there are a large number of pairwise comparisons in addition to cardinal scores generated from user votes on a small number of items ($n = 29$ items for each contest). Each pair of items received roughly 300 comparisons and each item also received roughly 200 cardinal votes. (The associated captions and visuals of the query types are given in~Figure~\ref{fig:nydueling}, and Figure~\ref{fig:508cardinal} in the supplementary material). Run directly on this dataset, Diffusion Rank Centrality did not show an appreciable advantage since the number of items was so small and hence similarity information provided less leverage over other methods.

\begin{figure}[h]
  \centering
  \includegraphics[width=.75\linewidth]{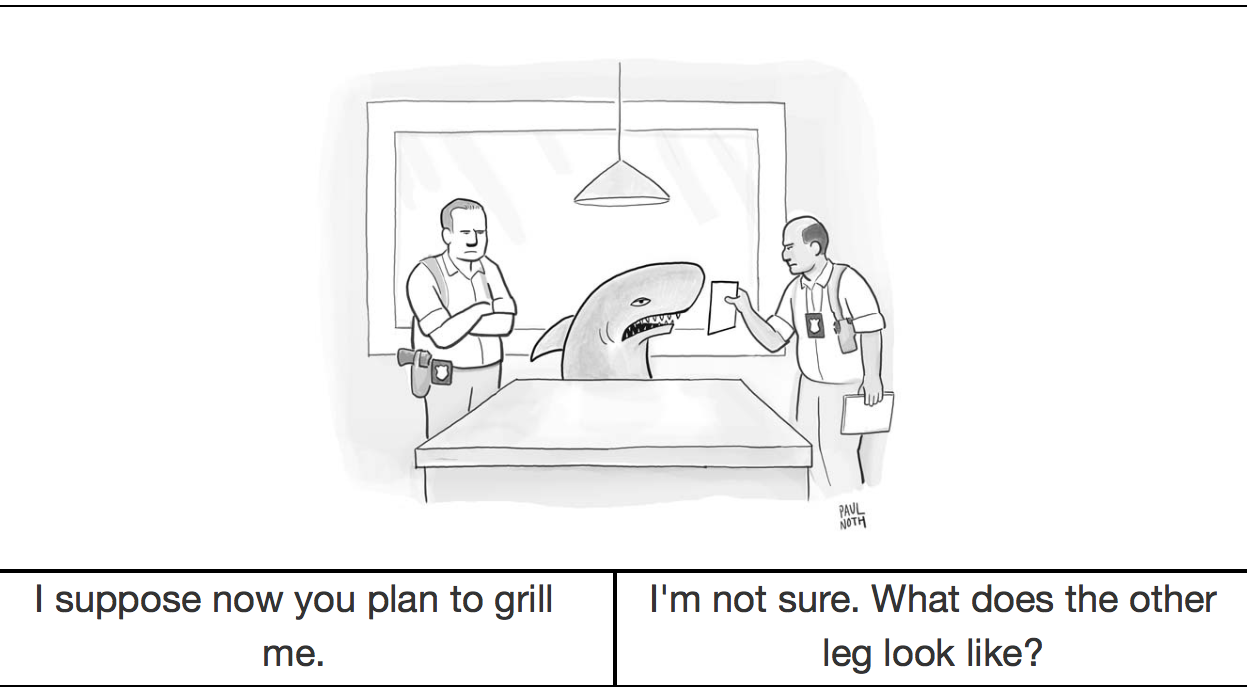}
  \caption{New Yorker Caption Competition Interface for pairwise comparisons for \#508. Users were asked to click on the caption they thought was funnier.}
  \label{fig:nydueling}
\end{figure}

\begin{figure}[h]
  \centering
  \includegraphics[width=.75\linewidth]{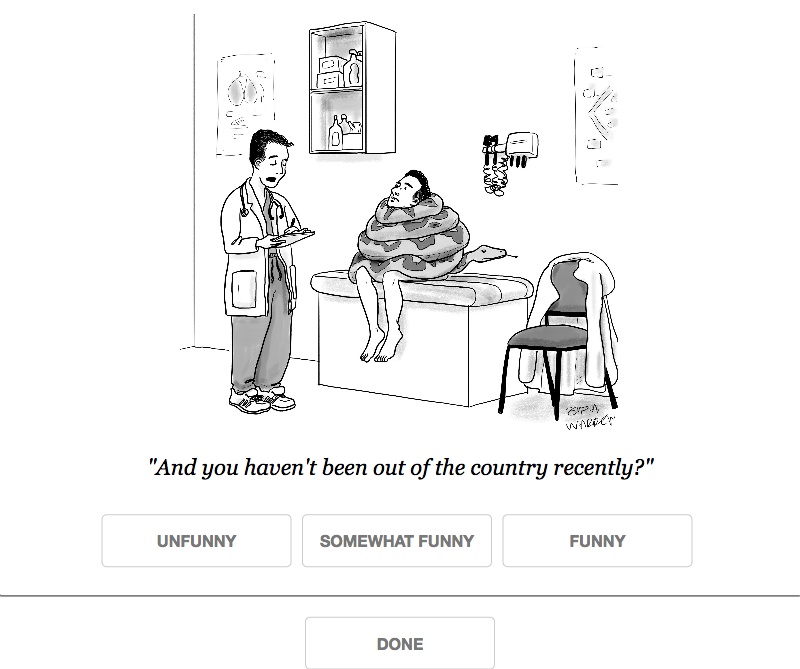}
  \caption{A sample of the voting user interface presented to readers of the New Yorker Magazine for contest \#651}
  \label{fig:ny651ui}
\end{figure}

\subsubsection{Cardinal Scores model BTL-scores}
\label{sec:cardscores}
We generate comparisons on a much larger set of captions for a different contest by transforming the cardinal data to infer pairwise comparisons. To determine this transformation, we used contest \#508 for which we had 300 pairwise comparisons and 200 cardinal votes. For each pair of captions $i,j$ in contest \#508, we compute $\hat{P}^{\text{emp}}_{ij}$, the empirical probability of item $i$ beating item $j$. In addition, we used the average empirical cardinal scores of items $i$ and $j$ denoted as $\hat{s}_i, \hat{s}_j$ we computed $\hat{P}^{\text{card}}_{ij} = \exp(\hat{s}_i)/(\exp(\hat{s}_i)+\exp(\hat{s}_j))$. In other words, we calculated the empirical probabilities implied by the cardinal scores and compared them to the empirical probabilities from the pairwise comparisons. A resulting scatterplot of the points $(\hat{P}^{\text{emp}}_{ij},\hat{P}^{\text{card}}_{ij})$ is shown in Figure~\ref{fig:phat_emp_vs_card}. Somewhat surprisingly, this plot demonstrates that a monotonic transformation of the cardinal scores seem to model an underlying pairwise probability model fairly well---implying that up to an exponential scaling transformation, the cardinal scores determine underlying BTL scores for the captions. This seems to be an interesting non-trivial result about ranking and humor that has not been previously observed.

\begin{figure}[h]
  \centering
  \includegraphics[width=.65\linewidth]{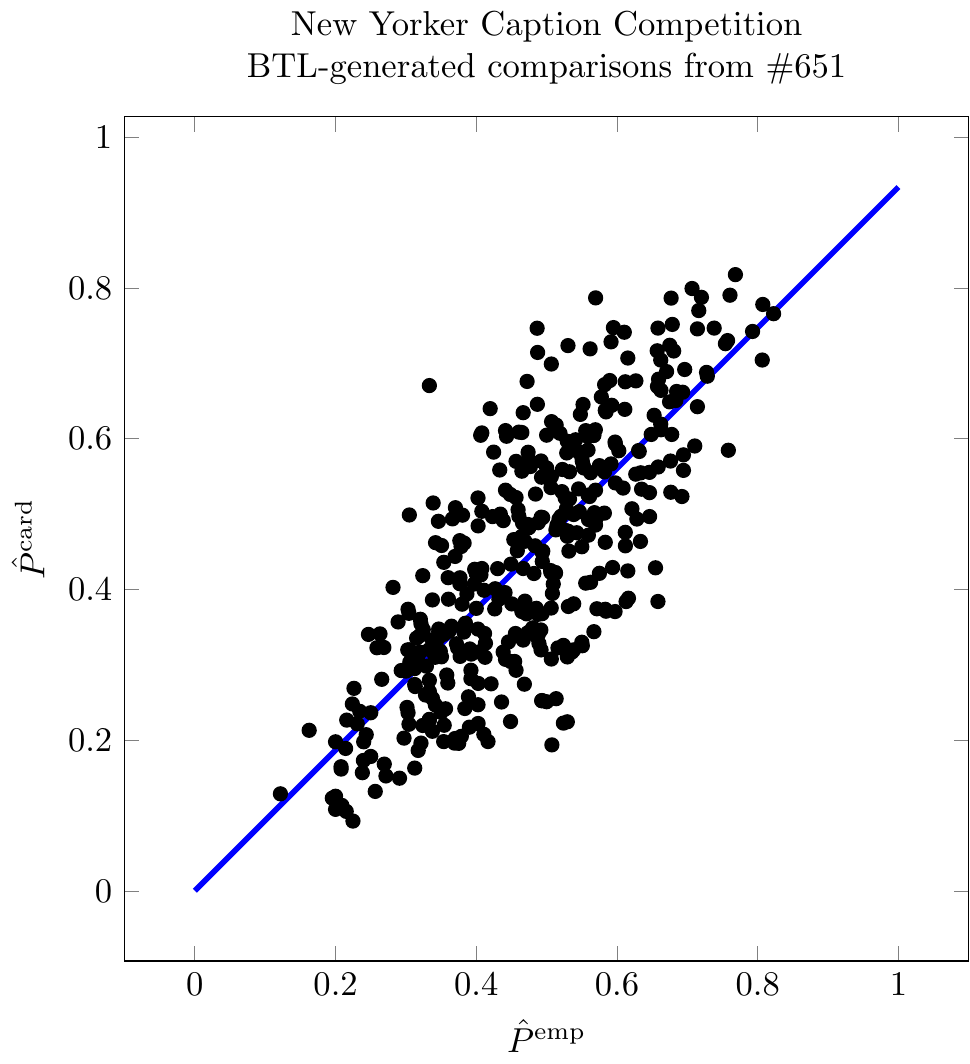}
  \caption{Scatter plot demonstrating the relationship between $\hat P^\text{emp}$ and $\hat P^\text{card}$.}
  \label{fig:phat_emp_vs_card}
\end{figure}

\

\subsubsection{Contest \#651}
Using the observations in the previous section, we chose a contest,  \#651, that did not have underlying pairwise comparisons but did have a large number of items all with cardinal scores. We then generated pairwise comparisons from these cardinal scores as described in Section~\ref{sec:cardscores}. The cartoon associated to this contest is in Figure~\ref{fig:ny651}.

More precisely, from the captions available, we took the 400 captions (out of roughly 7000) with largest empirical average cardinal score (each caption had around 250 votes) and generated BTL weights. %
We used the Universal Sentence Encoder in \cite{universalsentenceencoder} to generate 512 dimensional embeddings for each of the captions (this yields the additional structural information we need for regularization). The resulting plot contrasting the methods is shown in \ref{fig:ny651ui}, as before the kernel width was chosen on a validation set---in addition we used $(1-\tfrac{1}{\sqrt{m}})I + \frac{1}{\sqrt{m}}D^{(\sigma)}$ as the regularizer in Diffusion RankCentrality to debias the procedure.

In this setting, Diffusion RankCentrality performs extremely well, locking in a significantly better ranking almost immediately with few comparisons.

\begin{figure}[htb]
  \centering
  \includegraphics[width=\linewidth]{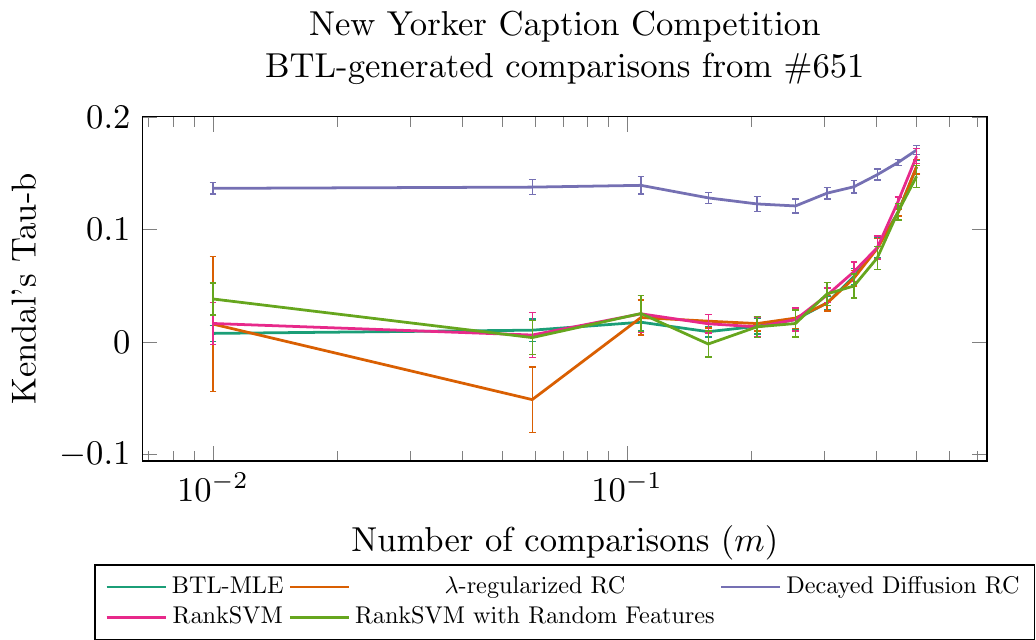}
  \caption{Test Error for various algorithms for the New Yorker Caption Competition \#651 with  $\sigma=.25$.}
  \label{fig:ny651}
\end{figure}

\subsection{Place Pulse}
Our final example involves comparisons arising from the Place Pulse dataset used in \cite{katariya2018adaptive}. There were 100 images of locations in Chicago in this dataset, and a total of 5750 comparisons where MTurk workers were asked which of the two locations they thought were safer.
We used ResNetV1 \cite{resnetv1} to generate features for the images of each location and broke the data up into a train, test and validation set (again used to select $\sigma$ and $\lambda$). Since we do not have an underlying ground truth ranking, we instead plot the test error in Figure \ref{fig:pp}.%

\begin{figure}[htb]
  \centering
  \includegraphics[width=\linewidth]{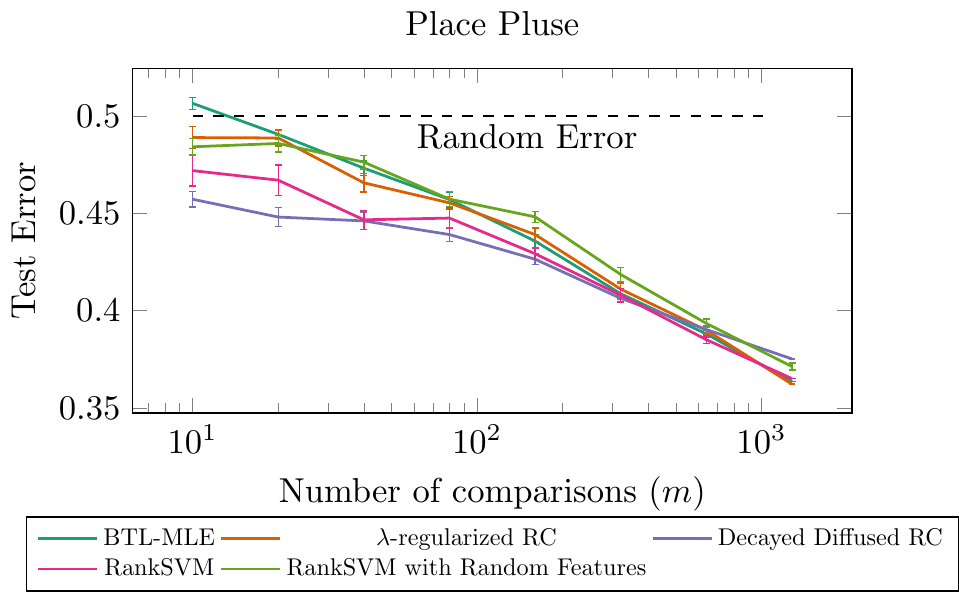}
  \caption{Performance of various algorithms from the Place Pulse dataset.}
  \label{fig:pp}
\end{figure}
Again, Diffusion RankCentrality (a non-classification based method) performed competitively matching the performance of RankSVM.

\section{Conclusion}
In this paper we provided a way to employ structure in the RankCentrality algorithm that provides meaningful results when data is scarce.
Along the way we provided a stronger sample complexity bound for a natural sampling scheme. For future work we hope to provide rigorous sample complexity bounds for diffusion based methods.

\subsubsection*{Acknowledgements}
The first and third authors were supported by the MIDAS Challenge Grant from the University of Michigan. The first author had the initial idea and motivation for this work while at Agero, Inc., and would like to thank Michael Bell.

\subsubsection*{References}
{\tiny\printbibliography[heading=none]}

\newpage
\onecolumn
\appendix
\appendixpage
\begin{table}[h]
  \centering
  \begin{tabular}{|c|p{5.5in}|}
    \hline
    Symbol      & Definition                                                                                                          \\
    \hline
    $\|\cdot\|$ & unless stated otherwise, vector norms are $\ell_2$ norms, and matrix norms are operator (spectral) norms            \\
    $\gamma$    & $\frac{n\mumin}{2(1+\sqrt{2})b^{3/2}}$                                                                              \\
    $w$         & stationary distribution of $Q$                                                                                      \\
    $\hat w$    & stationary distribution of $\hat Q$                                                                                 \\
    $\lambda$   & regularization constant, see $D_\lambda$                                                                            \\
    $\lmax(R)$  & second largest eigenvalue of matrix $R$ (because the largest eigenvalue of an irreducible Markov chain is always 1) \\
    $\mu_{ij}$  & probability that pair $(i,j)$ is observed                                                                           \\
    $\one$      & vector of all one entries, usually in $\R^n$                                                                        \\
    $b$         & $\max_{i,j} \frac{w_i}{w_j}$                                                                                        \\
    $k$         & number of comparisons per pair in sampling scheme in \cite{negahbanrc}                                              \\
    $n$         & number of items being compared                                                                                      \\
    $m$         & number of comparisons total                                                                                         \\
    $P$         & pairwise preference matrix                                                                                          \\
    $\hat P$    & empirical comparison matrix                                                                                         \\
    $Q$         & true markov chain (requires knowing $P$)                                                                            \\
    $\hat Q$    & empirical markov chain                                                                                              \\
    $D_\lambda$ & $(1-\lambda) I + \frac{\lambda}{n} \one \one^T$                                                                     \\
    \hline
  \end{tabular}
  \caption{Notation used in this paper.\label{tab:notation}}
\end{table}

\section{Convergence of RankCentrality}

Define
\begin{equation}
  Q^{(ij)} := e_{i}e_{j}^T - e_{i} e_{i}^T,
  \label{eq:Qij}
\end{equation}
and additionally
\begin{equation}
  Q_k =
  \begin{cases}
    Q^{(j_ki_k)} & \text{ if } y_k=0 \\
    Q^{(i_kj_k)} & \text{ if } y_k=1
  \end{cases}.\label{eq:Qk}
\end{equation}
We see now that
\begin{equation}
  \hat{Q} = I + \frac{1}{m} \sum_{k=1}^m Q_k,\label{eq:Qhatdefnsum}
\end{equation}
and for the remainder of our analysis we shall consider \eqref{eq:Qhatdefnsum} as the definition of $\hat Q$. Recall
\[Q_{ij} =
  \begin{cases}
    \mu_{ij} P_{ij}                    & \text{ if } i \neq j \\
    1 - \sum_{k\neq i} \mu_{ik} P_{ik} & \text{ if } i = j
  \end{cases},
\]
and observe that $\E(\hat Q) = Q$.

We begin our analysis of the RankCentrality algorithm by giving a bound on the spectral gap of the transition matrix $Q$ constructed from pairwise preferences.

\begin{prop}
  \label{prop:spectralgap}
  The spectral gap $1-\lmax$ of $Q$ is at least $\frac{n\mumin}{2b}$, where $b = \max_{i,j} \frac{w_i}{w_j}$.
\end{prop}

\begin{proof}
  We will use the following lemma from \cite[Lemma 6]{negahbanrc}.
  \begin{lem}[Comparison Inequality for Spectral Gaps\cite{negahbanrc}]
    Let $Q,\pi$ and $R,\tau$ be reversible Markov chains on a finite set $[n]$ representing random walks on a graph $G=([n],E)$, i.e.
    $R(i,j)=0$ and $Q(i,j)=0$ if $(i,j)\notin E$.
    For $\alpha\equiv\min_{(i,j)\in E}\{\pi_iQ_{ij}/\tau_iR_{ij}\}$ and
    $\beta\equiv\max_i\{\pi_i/\tau_i\}$,
    \[
      \frac{1-\lmax(Q)}{1-\lmax(R)} \geq \frac{\alpha}{\beta}
    \]
  \end{lem}
  We will invoke the above lemma with $R = \frac{1}{n} \one \one^T = [\frac{1}{n}]_{ij}$, $\tau = \frac{1}{n} \one = [\frac{1}{n}]_i$, $Q$ as we have defined it previously, and $\pi = w$. Observe that these define a reversible Markov chain. Since $R$ has rank 1, we have $\lmax(R) = 0$, which gives us that $1 - \lmax(Q) \geq \frac{\alpha}{\beta}$. Now we bound $\alpha$ and $\beta$.

  We have
  \begin{align*}
    \alpha = & \ \min_{i,j} \frac{w_iQ_{ij}}{\tau_iR_{ij}} = \min_{ij} \frac{w_i \mu_{ij} \frac{w_j}{w_i + w_j}}{\frac{1}{n} \frac{1}{n}} \geq  \ \min_{i,j} \frac{n^2\mumin w_iw_j}{(w_i+w_j)} \geq  \frac{n^2 \mumin \min_i w_i}{2}
  \end{align*}

  We also see $\beta = \max_i \frac{w_i}{\tau_i} = n \max_i w_i$. Thus, $\frac{\alpha}{\beta} \geq \frac{n\mumin}{2b}$.
\end{proof}

This bound is close to optimal when $\mu$ is uniform. Since the diagonal entries of $Q$ are each at least $1 - \frac{2}{n-1}$, we know $\frac{n-1}{2}(Q - (1-\frac{2}{n-1})I)$ is non-negative and row stochastic. By the Perron-Frobenius Theorem, the eigenvalues of $\frac{n-1}{2}(Q - (1-\frac{2}{n-1})I)$ lie in $[-1,1]$ and the eigenvalues of $Q$ must lie in $[1-\frac{4}{n-1}, 1]$. The difference between 1 and the smallest possible eigenvalue of $Q$ is only a factor of $4b$ larger than our bound on the spectral gap.

\begin{prop}[Effect of perturbing $Q$]
  \label{prop:Qperturbation}Let $Q$ be the true transition matrix as defined in \eqref{eq:Q}. For any ergodic Markov chain on $[n]$ with row-stochastic transition matrix $\tilde Q$ and stationary distribution $\tilde{w}$, if $\|Q - \tilde{Q}\| < \frac{n\mumin}{2 b^{3/2}}$, we have
  \[ \frac{\|\tilde{w} - w \|}{\|w\|} \leq \frac{2\|\Delta\|b^{3/2}}{n\mumin - 2\|\Delta\|b^{3/2}},\] where $\Delta = \tilde{Q} - Q$.
\end{prop}
\begin{proof}
  We begin by citing a lemma \cite[Lemma 2]{negahbanrc}.

  \begin{lem}
    \label{lem:markovperturbation}
    For any Markov chain $\tilde Q=Q+\Delta$ with a reversible Markov chain $Q$,
    let $p_t$ be the distribution of the Markov chain $\tilde Q$ when
    started with initial distribution $p_0$. Then,
    \begin{align*}
      \frac{\left\|p_t-w\right\|}{\|w\|} \leq \rho^t\frac{\|p_0-w\|}{\|w\|}\sqrt{\frac{w_{\rm max}}{w_{\rm min}}} + \frac{1}{1-\rho}\|\Delta\|_2\sqrt{\frac{w_{\rm max}}{w_{\rm min}}}\;.
    \end{align*}
    where $w$ is the stationary distribution of $Q$
    and $\rho=\lmax(Q)+\|\Delta\|_2\sqrt{w_{\rm max}/w_{\rm min}}$.
  \end{lem}

  As before, let $b = \max_{i,j} \frac{w_i}{w_j}$. Consider the limit as $t\to\infty$:
  \begin{itemize}
    \item when $0 \leq \rho < 1$ we have $\rho^t \to 0$, and
    \item when the Markov chain $\tilde Q$ is irreducible we have $p_t \to \tilde w$.
  \end{itemize}
  In this case,
  \begin{align*}
    \frac{\left\|\tilde w-w\right\|}{\|w\|} \leq \frac{1}{1-\rho}\|\Delta\|_2\sqrt{b}.
  \end{align*}

  Recall that $1 - \lmax(Q) > \frac{n\mumin}{2 b}$ by Proposition \ref{prop:spectralgap}. Now we have that $\rho < 1$ when $\|\Delta\| < \frac{n\mumin}{2b^{3/2}}$ because when this is the case, we have $\|\Delta\|\sqrt{b} < \frac{n\mumin}{2b}$ and hence $\rho \leq 1 - \frac{n\mumin}{2b} + \|\Delta\|\sqrt{b} < 1$.  Assuming $\|\Delta\| < \frac{n\mumin}{2b^{3/2}}$, we have
  \[\frac{\|\tilde w - w\|}{\|w\|} \leq \frac{\|\Delta\|\sqrt{b}}{\frac{n\mumin}{2b} - \|\Delta\|\sqrt{b}} = \frac{2\|\Delta\|b^{3/2}}{n\mumin - 2\|\Delta\|b^{3/2}}.\]
\end{proof}

For transition matrices $Q$ and $\hat Q$ we define the \emph{centered} transition matrices $Q'$ and $\hat Q'$ by subtracting $I$. That is, $Q' = Q - I$ and $\hat Q' = \hat Q - I$. These centered matrices $Q'$ and $\hat Q'$, as well as $Q_k$ and $Q^{(ij)}$ defined previously, have non-negative entries everywhere except on the diagonal (where they are non-positive) and their rows sum to zero. These centered matrices significantly simplify the algebra in the following computations.

\begin{lem}
  The difference $Z_k := \frac{Q_k - Q'}{m}$ is bounded in norm: $\|Z_k\| < \frac{3}{m}$.\label{lem:boundnormZ}
\end{lem}
\begin{proof}
  To bound $\|Q_k\|$, recall that $Q_k$ is of the form $Q^{(ij)} = (e_ie_j^T - e_i e_i)$.
  Observe that $ Q^{(ij)} Q^{(ij) T} = 2e_ie_i^T$. Therefore, $\|Q_k\| \leq \sqrt{2}$. By convexity of norms, $\|Q'\| = \|\E Q_k\| \leq \E \| Q_k\| \leq \sqrt{2}$. Using the triangle inequality we get $\| Q_k - Q'\| \leq 2\sqrt{2} < 3$.
\end{proof}

\begin{lem}
  \label{lem:rcvariance}
  Let $Z_k = \frac{Q_k - Q'}{m}$, as before. We can bound the variance term as:
  \[\sigma^2:= \max \left\{ \left \|\sum_{k=1}^m \E Z_kZ_k^*\right\|, \left \|\sum_{k=1}^m \E Z_k^*Z_k\right\| \right\} \leq \frac{3(n-1)\mumin}{m}.\]
\end{lem}
\begin{proof}
  To bound $\|\E Z_k Z_k^*\|$, we see
  \[\E Z_k Z_k^* = \frac{1}{m^2} \E \left( Q_k Q_k^{T} - Q_k Q^{\prime T} - Q' Q_k^{T} + Q' Q^{\prime T} \right) = \frac{1}{m^2} \E \left( Q_k Q_k^{T} - Q' Q^{\prime T} \right).\] We can compute these explicitly.

  Begin by considering the $Q_k Q_k^{T}$ term. We know $Q^{(ij)}Q^{(ij)T} = 2e_ie_i^T$. By simple algebra, we get $\E Q_k Q_k^{T} = \sum_{i} \sum_{j\neq i} 2\mu_{ij} P_{ji} e_i e_i^T$. Therefore, $\|\E Q_k Q_k^T \| \leq \max_i \sum_{j\neq i} 2 \mu_{ij} P_{ji} \leq 2 (n-1) \mumax$.%

  Computing $Q'Q^{\prime T}$ is more tedious.
  \begin{align*}
    Q' Q^{\prime T} = & \ \left( \sum_{i\neq j} \mu_{ij} P_{ij} (e_i e_j^T - e_i e_i^T) \right)\left( \sum_{u\neq v} \mu_{uv} P_{uv} (e_v e_u^T - e_u e_u) \right)    \\
    =                 & \sum_{i\neq j, u\neq v} \mu_{ij}\mu_{uv} P_{ij} P_{uv} (e_i e_j^T e_v e_u^T - e_i e_j^T e_u e_u - e_i e_i^T e_v e_u^T + e_i e_i^T e_u e_u^T).
  \end{align*}

  By ignoring zero terms (notice that the first of four summands is non-zero only when $j=v$, the second when $j=u$, etc.) and re-indexing, we get
  \begin{align*}
    Q' Q^{\prime T} = & \left( \sum_{i\neq \ell \neq j} \mu_{i\ell} \mu_{j\ell} P_{i\ell}P_{j\ell} e_i e_j^T - \sum_{i \neq j \neq \ell} \mu_{ij} \mu_{j\ell} P_{ij} P_{j\ell} e_i e_j^T - \sum_{j \neq i \neq \ell} \mu_{i\ell} \mu_{ji} P_{i\ell} P_{ji} e_i e_j^T + \sum_{u \neq i \neq v} \mu_{iu} \mu_{iv} P_{iu} P_{iv} e_i e_i^T \right),
  \end{align*}
  where statements such as $i\neq \ell \neq j$ mean $i \neq \ell$ and $j \neq \ell$ (but $i$ may be equal to $j$).
  This is a symmetric matrix, so its singular values are its eigenvalues. We can now invoke the Gershgorin circle theorem, a consequence of which is that $\|M\| < \max_i \sum_j |M_{ij}|$ for symmetric matrices. Therefore, $\|Q' Q^{\prime T}\| \leq 4n^2\mumax^2$. Finally, the triangle inequality gives $\|\E Z_k Z_k^* \| \leq \frac{1}{m^2} \left( 2(n-1)\mumax + 4n^2\mumax^2 \right)$.%

  We now turn to $ Z_k^* Z_k$.  Similar to the calculations above, simple algebra gets us
  \[\E Q_k^{T} Q_k = \sum_i \sum_{j\neq i} \mu_{ij}( P_{ij} + P_{ji} ) (e_ie_i^T - e_i e_j^T).\]
  As before, this is a symmetric matrix and we can use the Gershgorin circle theorem to give a bound on the largest singular value of $\E Q_k^T Q_k$:
  \[ \|\E Q_k^T Q_k\| \leq \max_i \sum_{j\neq i} 2\mu_{ij} \leq 2 (n-1)\mumax.\]

  As before computing $Q^{\prime T} Q'$ is more tedious but gives
  \begin{align*}
    Q^{\prime T} Q' = & \ \sum_{i \neq j} \sum_{u \neq v} \mu_{ij} \mu_{uv} P_{ij} P_{uv} (e_j e_i^T - e_i e_i^T)(e_u e_v^T - e_u e_u^T)                                                                                       \\
    =                 & \ \sum_{i \neq j} \sum_{u \neq v} \mu_{ij} \mu_{uv} P_{ij} P_{uv} (e_j e_i^T e_u e_v^T - e_j e_i^T e_u e_u^T - e_i e_i^T e_u e_v^T + e_i e_i^T e_u e_u^T)                                              \\
    =                 & \ \sum_{i \neq j} \sum_{v \neq i} \mu_{ij} \mu_{uv} P_{ij} P_{iv} (e_j e_v^T - e_j e_i^T - e_i e_v^T + e_i e_i^T)                                                                                      \\
    =                 & \ \sum_{i \neq j} \left( \sum_{\ell \neq i; \ell \neq j} \mu_{\ell i}\mu_{\ell j} P_{\ell i}P_{\ell j} - \mu_{ji} \mu_{j\ell} P_{ji}P_{j\ell} - \mu_{i\ell} \mu_{ij} P_{i\ell} P_{ij} \right)e_i e_j^T \\
                      & \qquad \qquad + \sum_{i} \left( \sum_{u\neq i,v\neq i} \mu_{iu} \mu_{iv} P_{iu} P_{iv} + \sum_{\ell \neq i} \mu_{\ell i} \mu_{\ell i} P_{\ell i}P_{\ell i} \right)e_i e_i^{T}.
  \end{align*}
  Again, we can invoke the Gershgorin circle theorem and see that $\|Q' Q^{\prime T}\| \leq 4n^2\mumax^2$. As before, the triangle inequality gives $\|\E Z_k^* Z_k \| \leq \frac{1}{m^2} \left( 2(n-1)\mumax + 4n^2\mumax^2 \right)$.%

  Finally, note that $Z_k$ are not only independent but also identically distributed and hence
  \[\max \left\{ \left \|\E \sum_k Z_k^* Z_k \right \|, \left \|\E \sum_k Z_k Z_k^* \right\| \right\} = m \max\left\{\|\E Z_k^*Z_k\|, \|\E Z_kZ_k^*\| \right\} \leq \frac{4(n-1)\mumax + 4n^2\mumax^2}{m}.\]
\end{proof}

We will soon need to use the Matrix Bernstein Inequality from \cite[Theorem 1.6]{tropptailbounds} and state it here as a lemma.
\begin{lem}[Matrix Bernstein \cite{tropptailbounds}]
  \label{lem:matbernstein}
  Consider a finite sequence $\{ \mathbf{Z}_k \}$ of independent, random matrices with dimensions $d_1 \times d_2$.  Assume that each random matrix satisfies
  \[
    \E \; \mathbf{Z}_k = \mathbf{0}
    \quad\text{and}\quad
    \norm{ \mathbf{Z}_k } \leq R
    \quad\text{almost surely}.
  \]
  Define
  \[
    \sigma^2 := \max\left\{
    \norm{ \sum\nolimits_k \E( \mathbf{Z}_k \mathbf{Z}_k^* ) }, \
    \norm{ \sum\nolimits_k \E(\mathbf{Z}_k^* \mathbf{Z}_k) }
    \right\}.
  \]
  Then, for all $t \geq 0$,
  \[
    \PP{\left(  \norm{ \sum\nolimits_k \mathbf{Z}_k } \geq t  \right)}
    \leq (d_1 + d_2) \cdot \exp\left( \frac{-t^2/2}{\sigma^2 + Rt/3} \right).
  \]
\end{lem}

Finally, we put this all together.

\begin{thm}[Convergence of Unregularized RankCentrality]
  Let $\hat Q$ be constructed as in \eqref{eq:Qhat}. If $\hat Q$ is ergodic and $\hat w$ is the stationary distribution of $\hat Q$, then we have (where probability is taken over the $m$ comparisons made under the BTL model and each pair is equally likely to get picked)  \[\PP\left(\frac{\|\hat w - w\|}{\|w\|} \leq \varepsilon\right) > 1 - 2n \exp \left( \frac{-\mumin^2\varepsilon^2 n m}{16b^3(1+\varepsilon)^2(\mumax + n\mumax^2)} \right).\]\label{thm:rcconvergence}
\end{thm}
\begin{proof}
  Assuming $\|\Delta\| < \frac{1}{nb^{3/2}}$, by Proposition \ref{prop:Qperturbation} we have \[\frac{\|\hat w - w\|}{\|w\|} \leq \frac{2\|\Delta\|b^{3/2}}{n\mumin - 2\|\Delta\|b^{3/2}}.\] This means we want \[\frac{2\|\Delta\|b^{3/2}}{n\mumin - 2\|\Delta\|b^{3/2}} < \varepsilon,\] which happens when $\|\Delta\| \leq \frac{\varepsilon n\mumin}{2b^{3/2}(1 + \varepsilon)}$. Note that this is stronger than $\|\Delta\| < \frac{n\mumin}{2b^{3/2}}$, so our previous assumption will hold.

  Finally, we let $t =  \frac{\varepsilon n\mumin}{2b^{3/2}(1 + \varepsilon)}$ and use Lemma \ref{lem:matbernstein} to get
  \[ \PP \left( \frac{\|\hat w - w\|}{\|w\|} \geq \varepsilon \right) \leq \PP \left( \|\hat Q - Q\| \geq t \right) \leq -2n\exp\left( \frac{-t^2}{\sigma^2 + Rt/3} \right),\]
  where we have $\sigma^2 \leq \frac{4(n-1)\mumax + 4n^2\mumax^2}{m}$ by Lemma \ref{lem:rcvariance} and $R < \frac{3}{m}$ by Lemma \ref{lem:boundnormZ}. Therefore, we get
  \begin{align*}
    \PP \left( \frac{\|\hat w - w\|}{\|w\|} \geq \varepsilon \right) & \leq  2n\exp\left(\frac{-\left( \frac{\varepsilon n\mumin}{2b^{3/2}(1+\varepsilon)} \right)^2}{\frac{4(n-1)\mumax + 4n^2\mumax^2}{m} +  \frac{\varepsilon n \mumin}{2mb^{3/2}(1+\varepsilon)} } \right) \\
                                                                     & \leq 2n \exp \left( \frac{-\mumin^2 \varepsilon^2 n^2 m}{4b^3(1 + \varepsilon)^2\left( 2n\mumax + 4n^2\mumax^2 \right) + 2b^{3/2}\varepsilon (1 + \varepsilon)n\mumin} \right)                          \\
                                                                     & \leq 2n \exp \left( \frac{-\mumin^2\varepsilon^2 n m}{16b^3(1+\varepsilon)^2(\mumax + n\mumax^2)} \right).
  \end{align*}
\end{proof}

\begin{cor}
  \label{cor:RCsamplecomp}
  Fix $\delta \in  (0, 1)$ and $\varepsilon \in (0, 1)$. If
  \[ m \geq 64b^3 n^{-1} \mumin^{-2} \varepsilon^{-2}(\mumax + n\mumax^2) \log \frac{2n}{\delta}  \]
  and the empirical Markov chain $\hat{Q}$ constructed as in \eqref{eq:Qhat} is ergodic, then with probability at least $1-\delta$, we have \[\frac{\|\hat w - w\|}{\|w\|} \leq \varepsilon.\]
\end{cor}
\begin{proof}
  We need
  \[ \PP\left( \frac{\|\hat w - w\|}{\|w\|} \geq \varepsilon \right) \leq 2n \exp \left( \frac{-\mumin^2\varepsilon^2 n m}{16b^3(1+\varepsilon)^2(\mumax + n\mumax^2)} \right) < \delta.\]
  By re-writing in terms of $m$, we see that the second inequality is true when
  \[ m > 16b^3(1+\varepsilon)^2 n^{-1} \mumin^{-2} \varepsilon^{-2}(\mumax + n\mumax^2) \log \frac{2n}{\delta} . \] The desired inequality now follows immediately from $\varepsilon < 1$ (we make this assumption for simplicity; the statement of the theorem is not very strong when $\varepsilon > 1$).
\end{proof}

When $\mu$ is uniform and $n>4$, the above theorem requires $m > 48 b^3 \varepsilon^{-2} n \log (\frac{2n}{\delta})$. We have given an $O\left( \varepsilon^{-2} n\log \frac{n}{\delta} \right)$ upper bound on the sample complexity. This is a much better bound than in \cite{agarwal14}. Their $O(\varepsilon^{-2} \mumin^{-2} n \log (\frac{n}{\delta}))$ scales as $O(\varepsilon^{-2} n^5\log(\frac{n}{\delta}))$ when $\mu$ is uniform and worse otherwise.

\section{Convergence of \texorpdfstring{$\lambda$}{Lambda}-Regularized RankCentrality}
\label{sec:regRCconv}

This section is devoted to an analysis of the bias-variance trade-off of $\lambda$-Regularized RankCentrality. We will compare
\begin{itemize}
  \item $\hat{\tilde w}$, the leading left eigenvector of $\hat QD_\lambda$, i.e., the output of $\lambda$-regularized RankCentrality, and
  \item $\tilde w$, the leading left eigenvector of $QD_\lambda$, i.e.,  the expected output of $\lambda$-regularized RankCentrality as $m\to\infty$,
  \item $w$, the leading left eigenvector of $Q$, and the expected output of RankCentrality as $m \to\infty$.
\end{itemize}

\begin{prop}[Regularized RankCentrality Bias]
  Fix $\lambda \in (0, \gamma)$. The asymptotic ($m\to \infty$) expectation of the output of the $\lambda$-Regularized RankCentrality algorithm is $\tilde w$ and the bias $ \|w - \tilde w\| / \|w\|$ can be bounded as
  \[ \frac{\|w - \tilde{w}\|}{\|w\|} \leq \frac{ \lambda}{\gamma - \lambda} \]\label{prop:regrcbias}
\end{prop}
\begin{proof}
  Let $\tilde{Q} = QD_\lambda$. We now have $Q - \tilde{Q} = \lambda ( \frac{1}{n}\one \one^T - Q)$ and $\|Q - \tilde{Q}\| \leq \lambda (1 + \sqrt{2})$. Now we apply Proposition \ref{prop:Qperturbation} to see that
  \[ \frac{\|w - \tilde{w}\|}{\|w\|} \leq \frac{2(1+\sqrt{2})\lambda b^{3/2}}{n\mumin - 2(1+\sqrt{2})\lambda b^{3/2}} = \frac{\lambda}{\gamma - \lambda}.\]
\end{proof}

\begin{thm}[Regularized RankCentrality]
  \label{thm:regrcconvergence}
  Fix $\lambda \in (0, \frac{\gamma}{2})$ and choose $\varepsilon  \in ( 2\lambda\gamma^{-1}, 1)$. We construct $\hat Q$ as before and let $\tilde{\hat{w}}$ be the stationary distribution (leading left eigenvector) of $\hat Q D_\lambda$ (i.e., the output of $\lambda$-regularized RankCentrality). We have
  \[
    \PP\left(\frac{\|\tilde{\hat{w}} - w\|}{\|w\|} < \varepsilon \right) >
    1 - 2n \exp \left( \frac{-(n\mumin \varepsilon - 4(1+\sqrt{2})b^{3/2}\lambda)^2m}{16b^3(1-\lambda)^2 \left( 4(n-1)\mumax + 4n^2 \mumax^2 \right)  + 4b^{3/2}(1-\lambda)(n\mumin \varepsilon - 4b^{3/2}(1+\sqrt{2})\lambda)} \right)
  \]
\end{thm}
\begin{proof}
  As we noted in the proof of Theorem \ref{thm:rcconvergence}, to guarantee $\|w - \tilde{\hat{w}}\|/\|w\| \leq \varepsilon$, we need $\|Q - \hat Q D_\lambda\| \leq \frac{\varepsilon n \mumin}{2(1 + \varepsilon)b^{3/2}}$. Using the triangle inequality, we have $\|Q - \hat Q D_\lambda\| \leq \|Q - QD_\lambda\| + \|QD_\lambda + \hat Q D_\lambda\|$. We showed in Proposition \ref{prop:regrcbias} that $\|Q - QD_\lambda\| \leq \lambda(1+ \sqrt{2})$. So we need
  \begin{align*}
    \|QD_\lambda - \hat Q D_\lambda\| & \ \leq \frac{\varepsilon n \mumin}{2(1 + \varepsilon)b^{3/2}} - \lambda(1+ \sqrt{2}) \leq \frac{\varepsilon n \mumin}{4b^{3/2}} - \lambda(1+ \sqrt{2}) \\
                                      & \ = \frac{(1+\sqrt{2})}{2}\varepsilon\gamma\ - \lambda (1+\sqrt{2}) = \frac{(1+\sqrt{2})}{2} (\varepsilon \gamma - 2\lambda)
  \end{align*}
  Note that this quantity is positive when $\varepsilon \in (2\lambda\gamma^{-1}, 1)$ (which is precisely the requirement in the hypothesis above). We have required that $\varepsilon < 1$ to simplify algebra; the theorem is not very useful otherwise. We now require that
  \[
    \|QD_\lambda - \hat Q D_\lambda\| \leq \frac{\varepsilon n \mumin}{4b^{3/2}} - \lambda(1+ \sqrt{2}).
  \]

  We can now invoke Lemma \ref{lem:matbernstein} with $Z_k = \frac{1}{m} (Q'D_\lambda - Q_kD_\lambda) = \frac{1}{m}(1-\lambda)(Q' - Q_k)$. By our previous calculations in Lemmas \ref{lem:boundnormZ} and \ref{lem:rcvariance},  we have the variance term $\sigma^2 \leq (1-\lambda)^2 \frac{4(n-1)\mumax + 4n^2\mumax^2}{m}$ and the norm term $R \leq (1-\lambda)\frac{3}{m}$. The resulting inequality is
  \begin{align*}
    \PP\left(\|QD_\lambda - \hat QD_\lambda\| \geq \frac{n\mumin \varepsilon}{4b^{3/2}} - (1+\sqrt{2})\lambda \right)
    \leq 2n\exp\left( \frac{-\left( \frac{n\mumin \varepsilon}{4b^{3/2}} - (1+\sqrt{2})\lambda \right)^2}{(1-\lambda)^2 \frac{4(n-1)\mumax + 4n^2\mumax^2}{m} + \frac{1-\lambda}{m} \left( \frac{n\mumin\varepsilon}{4b^{3/2}} - (1+\sqrt{2})\lambda \right)} \right),
  \end{align*}
  which simplifies to the desired inequality.
\end{proof}

\begin{cor}
  \label{cor:regRCsamplecomp}
  Recall $\gamma = \frac{n\mumin}{2(1+\sqrt{2})b^{3/2}}$. Let $\lambda \in (0, \frac{\gamma}{2})$. Choose $\delta \in  (0, 1)$ and  $\varepsilon  \in \left(2\lambda\gamma^{-1}, 1\right)$. If
  \[ m > \frac{68(1-\lambda)b^{3}(\mumax + n\mumax^2)}{n\mumin^2 \left( \varepsilon - 2\lambda\gamma^{-1} \right)^2} \log\frac{2n}{\delta}\]
  then with probability at least $1-\delta$, we have
  \[ \frac{\|\tilde{\hat{w}} - w\|}{\|w\|} \leq \varepsilon. \]
\end{cor}
\begin{proof}
  As in Corollary \ref{cor:RCsamplecomp}, we need
  \[
    \PP\left(\frac{\|\tilde{\hat{w}} - w\|}{\|w\|} > \varepsilon \right) < \delta,\]
  which we can guarantee when
  \[
    2n \exp \left( \frac{-(n\mumin \varepsilon - 4(1+\sqrt{2})b^{3/2}\lambda)^2m}{16b^3(1-\lambda)^2 \left( 4(n-1)\mumax + 4n^2 \mumax^2 \right)  + 4b^{3/2}(1-\lambda)(n\mumin \varepsilon - 4b^{3/2}(1+\sqrt{2})\lambda)} \right) < \delta.
  \]
  Rewriting in terms of $m$, we see that the second inequality is true when
  \begin{equation*}
    m > \frac{16b^3(1-\lambda)^2 \left( 4(n-1)\mumax + 4n^2 \mumax^2 \right)  + 4b^{3/2}(1-\lambda)(n\mumin \varepsilon - 4b^{3/2}(1+\sqrt{2})\lambda)}{(n\mumin \varepsilon - 4(1+\sqrt{2})b^{3/2}\lambda)^2} \log \frac{2n}{\delta}\\
  \end{equation*}

  The desired inequality now follows by replacing various terms in the above inequality with upper bounds for them (e.g., $(1-\lambda)^2 < 1 - \lambda$, $b^{3/2} < b^{3}$, and $\varepsilon < 1$).
\end{proof}

Empirical evidence suggests that values of $\lambda$ larger than $\frac{\gamma}{2}$ often yield meaningful results. Future work could include bridging this gap between the theory and application.

\section{Empirical Results: RankCentrality and \texorpdfstring{$\lambda$}{Lambda}-regularized Rankcentrality}
\label{sec:lambdaregempirical}

Our main experiments was to evaluate convergence of these algorithms with synthetic BTL scores and comparisons. We compared (unregularized) RankCentrality, $\lambda$-regularized RankCentrality (with $\lambda$ decaying as $\eta m^{-1/2}$ for different values of $\eta$, as described in Section \ref{sec:regRC}), the BTL maximum likelihood estimation (see equation \eqref{eq:btlmle}), and regularized BTL-MLE (using the Scikit-Learn \cite{scikit-learn} implementation of logistic regression).
The BTL score $w_i$ for each item $i$ was either
\begin{itemize}
  \item assigned by choosing $v_i$ uniformly at random from $[0, 5]$ and setting $w_i = \exp(v_i)$, or
  \item deterministically constructed, e.g., $w_i = i$ for $i \in [200]$.
\end{itemize}
Then, for various values of $m$, we generated $m$ comparisons (first chose $m$ pairs of items, uniformly at random from all possible pairs, then drew winners with probabilities according to the BTL model) and ran each algorithm on the same set of comparisons. In each of these cases, we record the $\ell_2$ error and the Kendall's Tau correlation metric. We repeat this process of generating comparisons and evaluating algorithms for a total of 40 times and record the mean and standard error of the $\ell_2$ error and the Kendall-Tau correlation metric. The results for some of these experiments are shown in Figure \ref{fig:regrcwithdecay}.%

\begin{figure}
  \begin{subfigure}[t]{0.49\linewidth}
    \centering
    \includegraphics[width=\linewidth]{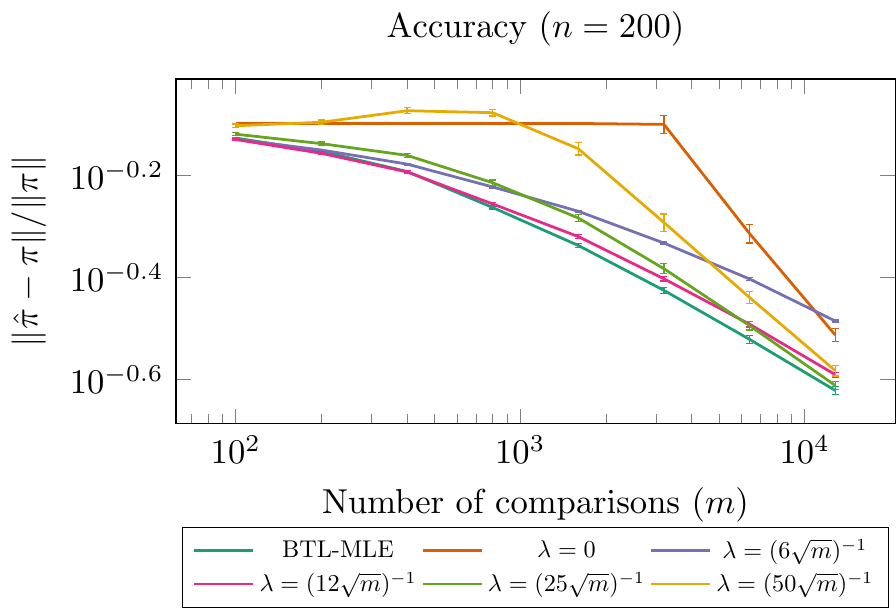}
    \caption{$w\in \R^{200}$ chosen at random.}
  \end{subfigure}
  \begin{subfigure}[t]{0.49\linewidth}
    \centering
    \includegraphics[width=\linewidth]{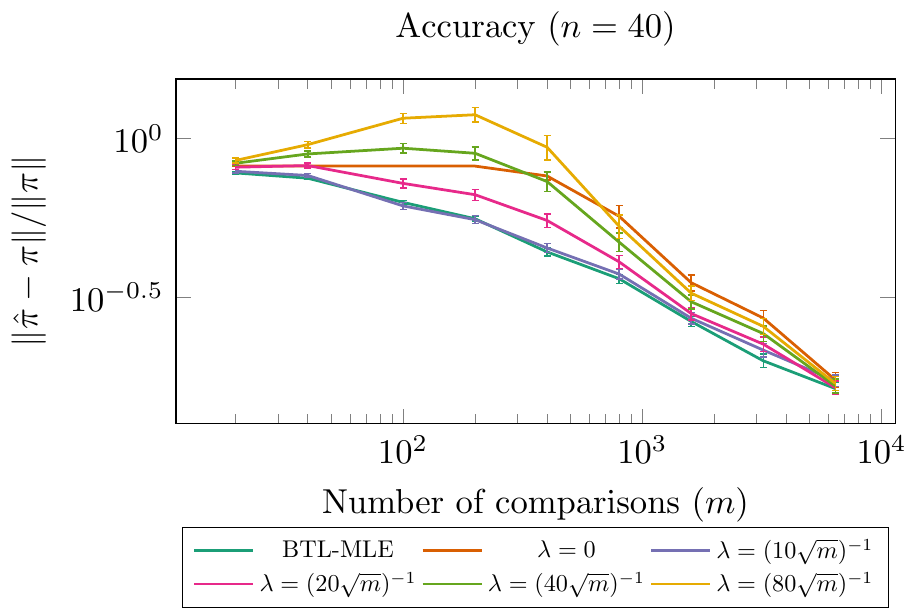}
    \caption{$w\in \R^{40}$ chosen at random.}
  \end{subfigure}
  \caption{Decaying $\lambda$ with a factor of $m^{-1/2}$.\label{fig:regrcwithdecay}}
\end{figure}

\section{Kendall's Tau-b}
\label{sec:kendalltau}

The Kendall-Tau correlation metric we use in our experiments is also know as Kendall's Tau-b, defined as
\begin{equation}
  \tau(\alpha, \beta) = \frac{P - Q}{\sqrt{(P + Q + T) * (P + Q + U)}},
  \label{eq:kendalltaub}
\end{equation}
where $P$ is the number of concordant pairs (i.e., the number of pairs $i,j$ such that the relative ordering of $\alpha_i$ and $\alpha_j$ is the same as that of $\beta_i$ and $\beta_j$), $Q$ the number of discordant pairs, $T$ the number of ties only in $\alpha$, and $U$ the number of ties only in $\beta$.

\begin{figure}
  \begin{subfigure}[t]{0.33\linewidth}
    \centering
    \includegraphics[width=\linewidth]{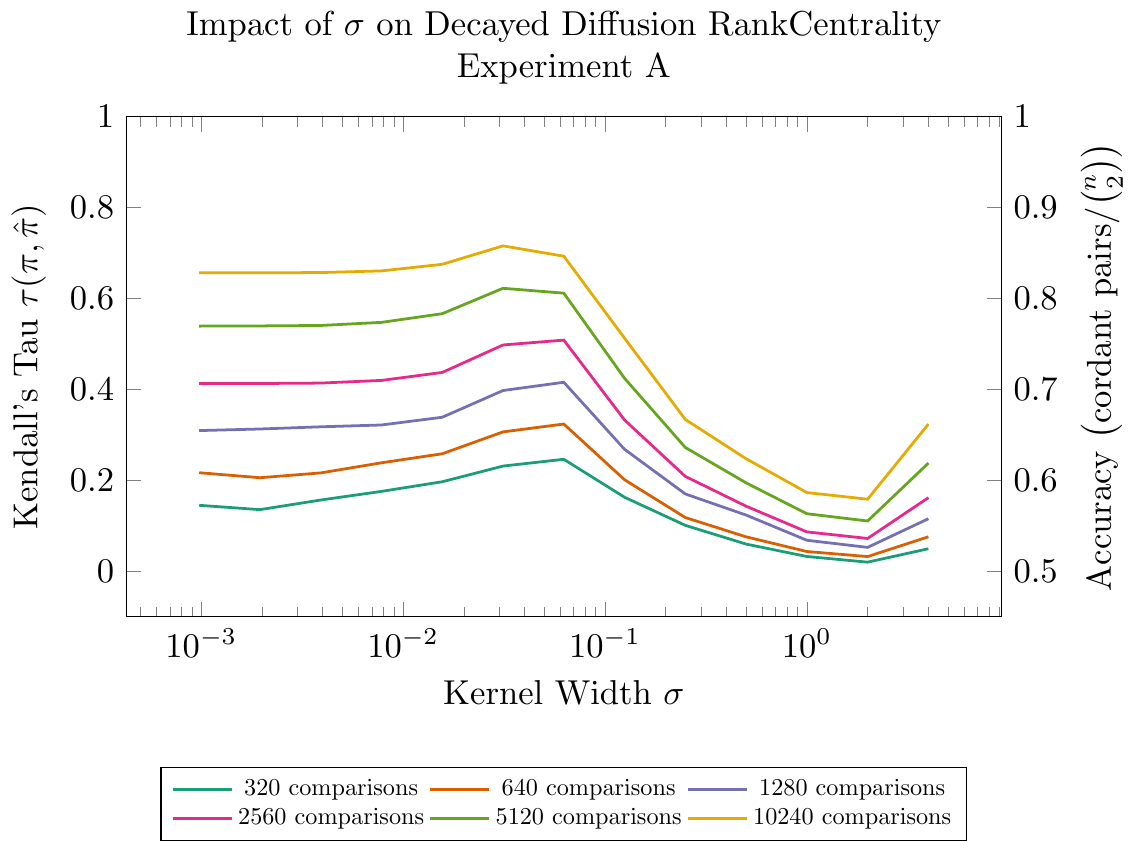}
    \caption{Synthetic Experiment B. \label{fig:exp2widths}}
  \end{subfigure}
  \begin{subfigure}[t]{0.33\linewidth}
    \centering
    \includegraphics[width=\linewidth]{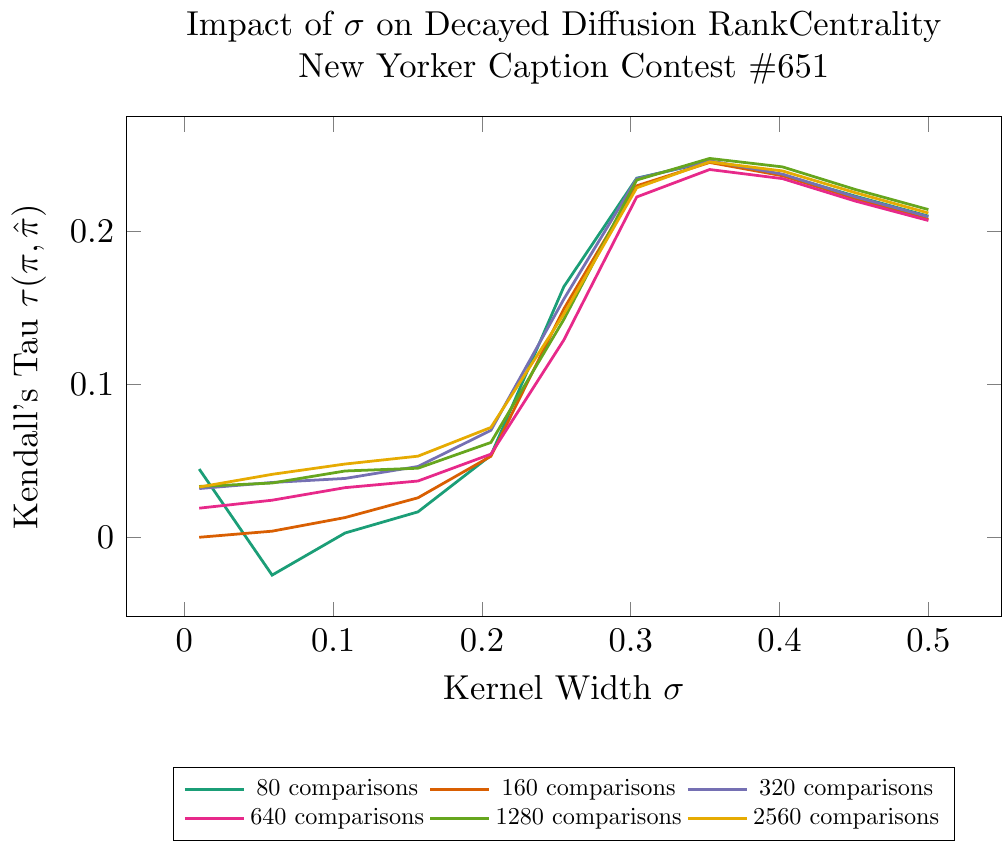}
    \caption{New Yorker Caption Competition \#651}
    \label{fig:ny651widths}
  \end{subfigure}
  \begin{subfigure}[t]{0.33\linewidth}
    \centering
    \includegraphics[width=\linewidth]{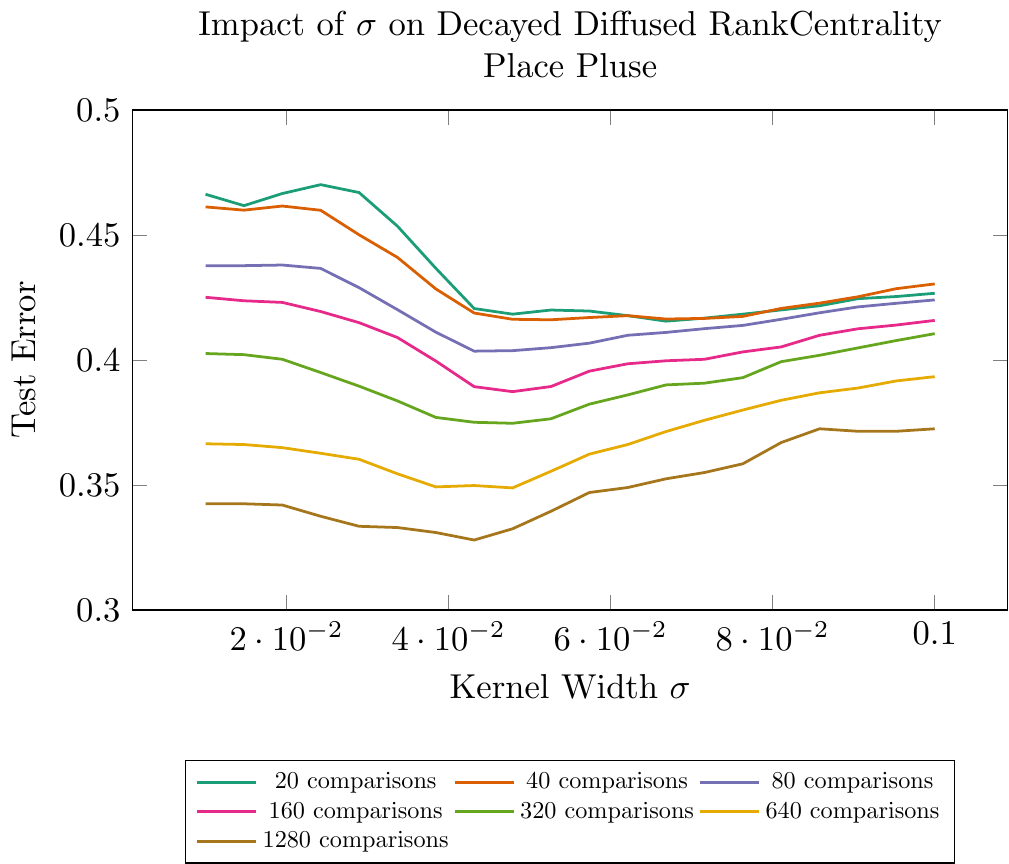}
    \caption{Place Pulse dataset.}
    \label{fig:ppwidths}
  \end{subfigure}
  \caption{Impact of kernel width on Diffusion RankCentrality for various datasets.}
\end{figure}

\section{New Yorker Caption Contest}
\label{sec:nycc}

\begin{figure}[h]
  \centering
  \includegraphics[width=.75\linewidth]{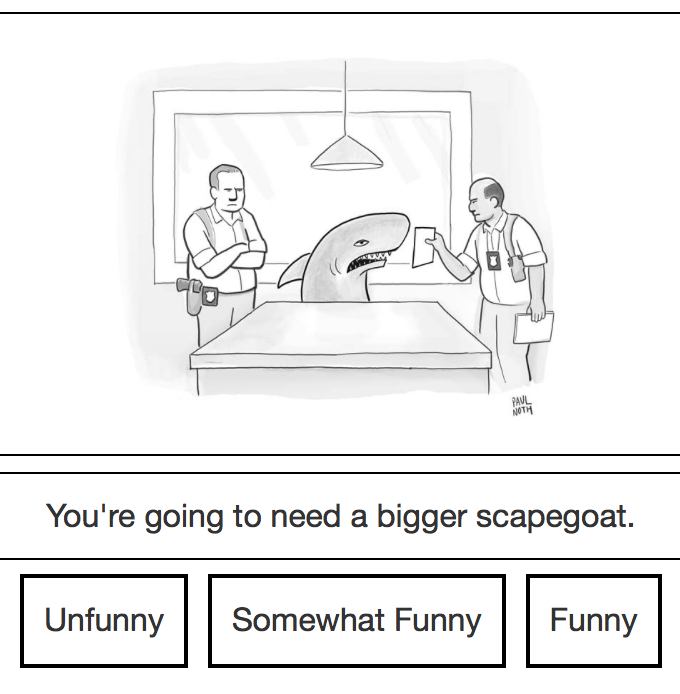}
  \caption{New Yorker Caption Competition Interface for pairwise comparisons for 508. Users were asked to vote for each caption.}
  \label{fig:508cardinal}
\end{figure}

\end{document}